\documentclass[a4paper,UKenglish,cleveref, autoref, thm-restate]{lipics-v2021}
\crefname{app}{Appendix}{Appendices}

\newtheorem{assumption}[theorem]{Assumption}
\theoremstyle{definition}
\newtheorem{notation}[theorem]{Notation}

\usepackage{tikz,tikz-cd}
\usepackage{comment}
\usepackage{centernot}
\usepackage{mathtools,amsmath,amssymb}
\usepackage{tikzit}			
\newcommand{\no}[1]{#1^{\scriptscriptstyle \bot}} 

\tikzstyle{morphism}=[fill=white, draw=black, shape=rectangle]
\tikzstyle{generic morphism}=[fill=white, draw=black, shape=rectangle, dashed]
\tikzstyle{small box}=[fill=white, draw=black, shape=rectangle, minimum width=0.5cm, minimum height=0.5cm]
\tikzstyle{medium box}=[fill=white, draw=black, shape=rectangle, minimum width=0.5cm, minimum height=0.8cm]
\tikzstyle{large morphism}=[fill=white, draw=black, shape=rectangle, minimum width=0.5cm, minimum height=1.2cm]
\tikzstyle{bn}=[fill=black, draw=black, shape=circle, inner sep=1.5pt]
\tikzstyle{bw}=[fill=white, draw=black, shape=circle, inner sep=1.5pt]
\tikzstyle{bin}=[fill=white, draw=black, shape=circle, inner sep=0pt]
\tikzstyle{effect}=[fill=white, draw=black, regular polygon, regular polygon sides=3, minimum width=0.8cm, inner sep=0pt]
\tikzstyle{state}=[fill=white, draw=black, regular polygon, regular polygon sides=3, minimum width=0.8cm, shape border rotate=90, inner sep=0pt]
\tikzstyle{medium state}=[fill=white, draw=black, regular polygon, regular polygon sides=3, minimum width=1.3cm, inner sep=0pt, shape border rotate=90]
\tikzstyle{large state}=[fill=white, draw=black, regular polygon, regular polygon sides=3, minimum width=2.2cm, shape border rotate=180, inner sep=0pt]
\tikzstyle{wide state}=[fill=white, draw=black, shape=isosceles triangle, minimum width=0.8cm, shape border rotate=270, inner sep=1.4pt, minimum height=0.5cm, isosceles triangle apex angle=80]
\tikzstyle{wn}=[fill=white, draw=black, shape=circle, inner sep=1.5pt]
\tikzstyle{blue morphism}=[fill=white, draw={rgb,255: red,15; green,0; blue,150}, shape=rectangle, text={rgb,255: red,15; green,0; blue,150}, tikzit category=blue]
\tikzstyle{red morphism}=[fill=white, draw={rgb,255: red,150; green,0; blue,2}, shape=rectangle, text={rgb,255: red,150; green,0; blue,2}, tikzit category=red]
\tikzstyle{blue state}=[fill=white, draw={rgb,255: red,15; green,0; blue,150}, shape=circle, regular polygon, regular polygon sides=3, minimum width=0.8cm, shape border rotate=180, inner sep=0pt, text={rgb,255: red,15; green,0; blue,150}, tikzit category=blue]
\tikzstyle{blue node}=[fill={rgb,255: red,15; green,0; blue,150}, draw={rgb,255: red,15; green,0; blue,150}, shape=circle, tikzit category=blue, inner sep=1.5pt]
\tikzstyle{blue}=[text={rgb,255: red,15; green,0; blue,150}, tikzit draw={rgb,255: red,191; green,191; blue,191}, tikzit category=blue, tikzit fill=white, inner sep=0mm]
\tikzstyle{blue wide state}=[fill=white, draw={rgb,255: red,15; green,0; blue,150}, text={rgb,255: red,15; green,0; blue,150}, shape=isosceles triangle, minimum width=0.8cm, shape border rotate=270, inner sep=1.4pt, minimum height=0.5cm, isosceles triangle apex angle=80]
\tikzstyle{red node}=[fill={rgb,255: red,150; green,0; blue,2}, draw={rgb,255: red,150; green,0; blue,2}, shape=circle, inner sep=1.5pt]
\tikzstyle{Purple node}=[fill={rgb,255: red,120; green,0; blue,120}, draw={rgb,255: red,120; green,0; blue,120}, text={rgb,255: red,120; green,0; blue,120}, shape=circle, inner sep=1.5pt]
\tikzstyle{red}=[text={rgb,255: red,150; green,0; blue,2}, inner sep=0mm, tikzit fill=white, tikzit draw={rgb,255: red,191; green,191; blue,191}]
\tikzstyle{purple}=[text={rgb,255: red,150; green,0; blue,150}, inner sep=0mm, tikzit fill=white, tikzit draw={rgb,255: red,191; green,191; blue,191}]
\tikzstyle{white morphism}=[fill=white, draw=white, shape=rectangle, tikzit draw={rgb,255: red,139; green,139; blue,139}]
\tikzstyle{leak morphism}=[fill=white, draw={rgb,255: red,120; green,0; blue,85}, shape=rectangle, text={rgb,255: red,120; green,0; blue,85}, tikzit category=leak]
\tikzstyle{leak}=[text={rgb,255: red,120; green,0; blue,85}, inner sep=0mm, tikzit fill=white, tikzit draw={rgb,255: red,191; green,191; blue,191}, tikzit category=leak]
\tikzstyle{leak node}=[fill={rgb,255: red,120; green,0; blue,85}, draw={rgb,255: red,120; green,0; blue,85}, shape=circle, inner sep=1.5pt, tikzit category=leak]
\tikzstyle{horiz state}=[fill=white, draw=black, regular polygon, regular polygon sides=3, minimum width=1cm, shape border rotate=90, inner sep=0pt]
\tikzstyle{none_90}=[rotate=90]
\tikzstyle{none_-90}=[rotate=-90]

\tikzstyle{arrow}=[->]
\tikzstyle{dashed box}=[-, dashed]
\tikzstyle{blue arrow}=[-, draw={rgb,255: red,15; green,0; blue,150}, tikzit category=blue]
\tikzstyle{red arrow}=[-, draw={rgb,255: red,150; green,0; blue,2}, tikzit category=red]
\tikzstyle{purple arrow}=[->, draw={rgb,255: red,120; green,0; blue,120}, >=stealth, shorten <=2pt, shorten >=2pt]
\tikzstyle{protected purple arrow}=[->, draw={rgb,255: red,120; green,0; blue,120}, >=stealth, shorten <=2pt, shorten >=2pt, preaction={line width=1.8pt, white, draw}]
\tikzstyle{mapsto}=[{|->}]
\tikzstyle{double wire}=[-, double]
\tikzstyle{curly brace}=[-, draw=none, tikzit draw={rgb,255: red,128; green,0; blue,128}]
\tikzstyle{protected}=[-, preaction={line width=1.8pt,white,draw}]
\tikzstyle{leak arrow}=[-, tikzit draw={rgb,255: red,150; green,0; blue,120}]
\tikzstyle{protected leak arrow}=[-, tikzit draw={rgb,255: red,150; green,0; blue,120}]
\tikzstyle{hollow arrow}=[-, very thin, white, preaction={line width=0.7pt,draw={rgb,255: red,120; green,0; blue,85}}, tikzit category=leak, tikzit draw={rgb,255: red,150; green,0; blue,120}]
\tikzstyle{protected hollow arrow}=[-, very thin, white, preaction={line width=0.7pt,draw={rgb,255: red,120; green,0; blue,85},preaction={line width=2.1pt,white,draw}}, tikzit category=leak, tikzit draw={rgb,255: red,150; green,0; blue,120}]
\tikzstyle{over arrow}=[-, black, preaction={draw=white, double}]
\tikzstyle{curly brace}=[-, decorate, decoration={brace,amplitude=5pt}]
\tikzstyle{inv curly brace}=[-, decorate, decoration={brace,amplitude=5pt,mirror}]

\usetikzlibrary{decorations.pathreplacing,decorations.pathmorphing}	
\usepackage{tikz-cd}
\usepackage{bm}				

\usepackage{todonotes}

\makeatletter
\def\cref@thmoptarg[#1]#2#3#4{%
	    \ifhmode\unskip\unskip\par\fi%
	    \normalfont%
	    \trivlist%
	    \let\thmheadnl\relax%
	    \let\thm@swap\@gobble%
	    \thm@notefont{\fontseries\mddefault\upshape}%
	    \thm@headpunct{.}
	    \thm@headsep 5\p@ plus\p@ minus\p@\relax%
	    \thm@space@setup%
	    #2
	    \@topsep \thm@preskip               
	    \@topsepadd \thm@postskip           
	    \def\@tempa{#3}\ifx\@empty\@tempa%
	      \def\@tempa{\@oparg{\@begintheorem{#4}{}}[]}%
	    \else%
	      \refstepcounter[#1]{#3}
	      \@namedef{cref@#3@alias}{#1}
	      \def\@tempa{\@oparg{\@begintheorem{#4}{\csname the#3\endcsname}}[]}%
	    \fi%
	    \@tempa}%
\makeatother

\newcommand{\newterm}[1]{\emph{\textbf{#1}}}

\renewcommand{\emptyset}{\varnothing} 
\newcommand{\longmapsfrom}{\mathrel{\reflectbox{$\longmapsto$}}}

\newcommand{\cat}[1]{{\mathsf{#1}}}

\newcommand{\id}{\mathrm{id}} 		
\newcommand{\tensor}{\otimes}

\newcommand{\comp}{ 		
	\mathchoice{\,}{\,}{}{} 	
}

\DeclareMathOperator{\cop}{copy}
\DeclareMathOperator{\del}{del}
\DeclareMathOperator{\compare}{compare}

\newcommand{\copycomp}[1]{\operatorname{Cpy}(#1)}
\newcommand{\compcomp}[1]{\operatorname{Cmp}(#1)}
\newcommand{\In}[1]{\operatorname{In}_{#1}}
\newcommand{\Out}[1]{\operatorname{Out}_{#1}}
\newcommand{\cC}{\mathsf{C}}		
\newcommand{\graph}[1]{\operatorname{gr}(#1)}
\newcommand{\freehyp}[1]{\mathsf{FreeHyp}(#1)}
\newcommand{\freecdo}[1]{\mathsf{FreeCD}(#1)}
\newcommand{\syn}[1]{\mathsf{HSyn}_{#1}} 
\newcommand{\cdsyn}[1]{\mathsf{CDSyn}_{#1}}
\newcommand{\cartsyn}[1]{\mathsf{FreeC}_{#1}} 
\newcommand{\parents}[1]{\operatorname{Pa}(#1)}
\newcommand{\mor}[1]{\operatorname{Mor}(#1)}
\newcommand{\tr}[1]{\operatorname{Tr}(#1)}

\newcommand{\h}{\mathcal{H}}
\newcommand{\g}{\mathcal{G}}

\newcommand{\finstoch}{\mathsf{FinStoch}}
\newcommand{\odag}{\cat{ODAG}}
\newcommand{\ougr}{\cat{OUGr}}
\newcommand{\cdcat}{\cat{CDCat}}
\newcommand{\hypcat}{\cat{HypCat}}
\newcommand{\bn}{\cat{BN}}
\newcommand{\mn}{\cat{MN}}
\newcommand{\finprojstoch}{\mathsf{FinProjStoch}}

\newcommand{\mat}[1][\mathbb{R}^{\geq 0}]{\mathsf{Mat}(#1)}

\newcommand{\as}[1]{
	\def\relstate{#1}%
	\ifx\relstate\empty
		\text{a.s.}%
	\else
		{#1\text{-a.s.}}%
	\fi
}

\newcommand{\clique}[1]{C\ell(#1)}

\providecommand{\given}{\,|\,}			
\newcommand{\edge}{\mathrel{\rule[0.5ex]{1em}{0.2pt}}}

\title{An Algebraic Approach to Moralisation and Triangulation of Probabilistic Graphical Models {\Large \colorbox{lipicsYellow}{Full Version}}}
\titlerunning{An Algebraic Approach to Moralisation and Triangulation of PGMs, Full Version} 

\author{Antonio {Lorenzin}}{University College London, Computer Science department, UK \and \url{https://sites.google.com/view/antonio-lorenzin} }{a.lorenzin@ucl.ac.uk}{https://orcid.org/0000-0002-2415-4261}{}

\author{Fabio Zanasi}{University College London, Computer Science department, UK \and \url{http://www.zanasi.com/fabio} }{f.zanasi@ucl.ac.uk}{https://orcid.org/0000-0001-6457-1345}{}

\authorrunning{A. Lorenzin and F. Zanasi} 

\Copyright{Antonio Lorenzin and Fabio Zanasi} 



\category{} 

\funding{Our work has been supported by the ARIA Safeguarded AI TA1.1 programme.}

\nolinenumbers 

\EventEditors{Corina C\^{i}rstea and Alexander Knapp}
\EventNoEds{2}
\EventLongTitle{11th Conference on Algebra and Coalgebra in Computer Science (CALCO 2025)}
\EventShortTitle{CALCO 2025}
\EventAcronym{CALCO}
\EventYear{2025}
\EventDate{June 16--18, 2025}
\EventLocation{University of Strathclyde, UK}
\EventLogo{}
\SeriesVolume{342}
\ArticleNo{6}

\begin{document}

\maketitle

\begin{abstract}
	Moralisation and Triangulation are transformations allowing to switch between different ways of factoring a probability distribution into a graphical model. Moralisation allows to view a Bayesian network (a directed model) as a Markov network (an undirected model), whereas triangulation works in the opposite direction. We present a categorical framework where these transformations are modelled as functors between a category of Bayesian networks and one of Markov networks. The two kinds of network (the objects of these categories) are themselves represented as functors, from a `syntax' domain to a `semantics' codomain. Notably, moralisation and triangulation are definable inductively on such syntax, and operate as a form of functor pre-composition. This approach introduces a modular, algebraic perspective in the theory of probabilistic graphical models. \end{abstract}


\section{Introduction}

Increasingly in recent years, category-theoretic semantics has been adopted to identify the algebraic structures underpinning probabilistic computation. These studies have interleaved with the theory of probabilistic programming, Bayesian inference, and machine learning, see e.g.~\cite{Barthe_Katoen_Silva_2020,lorenz2023causalmodels,shiebler2021categorytheorymachinelearning} for an overview. The overarching goal is to establish a principled, mathematically rigorous semantics of probabilistic reasoning, bringing formal clarity to traditional approaches. One notable example is the work \cite{jacobs2019mathematics}, which employs categorical methods to classify different way of reasoning with soft evidence. A key strength of category theory in this context is its emphasis on abstraction: the categorical perspective enables a unified treatment of different probabilistic frameworks---such as discrete, measure-theoretic, and Gaussian models---allowing for generic definitions of conditioning, independence, marginals, and related concepts. A prominent programme in this direction is that of \emph{Markov categories}~\cite{fritz2019synthetic}. Another major advantage of categorical methods is their inherent compositionality. By leveraging categorical structures, we gain insight into how the meaning of probabilistic computation can be systematically decomposed into more fundamental components.

The focus of this work is the categorical semantics of probabilistic graphical models (PGMs). PGMs provide a structured representation of probability distributions, making them useful for a variety of tasks in machine learning, statistics, and artificial intelligence. The `structure' is expressed by a combinatorial object (usually a graph), whose vertices represent events, and edges capture probabilistic dependencies between them. Two of the most prevalent examples of PGMs are Bayesian networks, which are built on directed acyclic graphs, and Markov networks (also called Markov random fields), based on undirected graphs. These models have orthogonal expressivity, and shine in different applications. Therefore, the theory of PGMs comprises procedures transforming a Bayesian network into a Markov network, called \emph{moralisation}, and a Markov network into a Bayesian network, called \emph{triangulation}. Importantly, these transformations do not introduce new conditional independencies between events, but may lose some of those expressed by the original network. Moralisation and triangulation ubiquitously appear in a variety of tasks in Bayesian reasoning, such as the junction tree algorithm, clique tree message passing,  variable elimination, and graph-based optimisation--- see~\cite{koller2009probabilistic,barberBRML2012} for an overview.

In traditional presentations of PGMs, the combinatorial representation of dependencies between variables is deeply entwined with their probabilistic meaning, variously expressed in terms of conditional probability tables (for Bayesian networks), or factors (for Markov networks). This heterogeneity reflects on the way moralisation and triangulation are defined: even though they exclusively affect the combinatorial structure of the network, strictly speaking they need to be applied to the whole network, including the probabilistic component. Moreover, formal methods for reasoning about the combinatorial structure are limited, compared to e.g.\ algebraically/syntactically defined objects, which hampers a systematic study of the mathematical properties of moralisation and triangulation.

In this paper, we develop a categorical semantics of Bayesian networks, Markov networks, and the translating operations of moralisation and triangulation. Our approach draws inspiration from Lavwere's framework of functorial semantics~\cite{lawvere1963functorial}, which understands an algebraic theory $T$ via its freely generated cartesian category of terms $\cartsyn{T}$, and a model of $T$ as product-preserving functors $\cartsyn{T} \to \cat{C}$ to a cartesian category $\cat{C}$. While Lawvere's framework identifies cartesianity (the existence of finite products) as the core categorical structure of algebraic theories, categorical approaches to Bayesian networks, beginning with~\cite{fong2013causaltheoriescategoricalperspective,JacobsZ16,jacobs2019causal_surgery}, have identified in \emph{copy-delete (CD) categories}~\cite{CorradiniG99,fong2013causaltheoriescategoricalperspective,JacobsZ16,Fritz_2023} the basic structure for expressing the connectivity of DAGs. In particular, \cite{jacobs2019causal_surgery} shows a 1-to-1 correspondence between Bayesian networks on a DAG $\g$ and structure-preserving functors from $\cdsyn{\g}$, the free CD-category built from the data of $\g$, to $\finstoch$, the CD-category of stochastic matrices. In other words, Bayesian networks are the models of the `theory' $\g$ in $\finstoch$. Also, morphisms of $\cdsyn{\g}$ are represented as \emph{string diagrams}~\cite{selinger11graphical,piedeleuzanasi}, to emphasise their interpretation as a two-dimensional syntax for graphs.

The functorial semantics perspective on Bayesian network is the starting point of our approach. We will expand it to encompass Markov networks, moralisation, and triangulation. At a conceptual level, our aims are two-fold:
\begin{itemize}
	\item By regarding both Bayesian and Markov networks as models of a freely generated category into a category of stochastic processes, we cleanly separate their `syntax' (the combinatorial structure) from their `semantics' (the probabilistic interpretation). This allows to reason separately on each component, enabling a deeper understanding of the algebraic structures underpinning probabilistic graphical models. 
	\item Moralisation and triangulation become examples of this paradigm. We are able to interpret them as transformations acting on the syntax of the network, defined simply by functor pre-composition. Furthermore, being functors from a freely generated category, they can be defined inductively on the syntax generators. In this way, we have reduced the traditional \emph{combinatorial} perspective on networks to a purely \emph{syntactic} perspective, which is more amenable to modularisation and algebraic reasoning.
\end{itemize}
We now outline our contributions, also indicating how they appear in the paper structure.
\begin{itemize}
	\item After recalling copy-delete and hypergraph categories (\cref{sec:cd-hypergraphcats}), in \cref{sec:networks} we characterise Bayesian and Markov networks as functors. For Markov networks, we give a functorial semantics characterisation analogous to the one of Bayesian networks described above. The key difference is that, rather than CD-categories, we need the richer structure of hypergraph categories~\cite{CARBONI198711,fong2019hypergraph,marsden_et_al:17} in order to account for the lack of directionality of Markov networks. We also provide a functorial characterisation of `irredundant' networks, this time novel both for the Bayesian and Markov case. Intuitively, irredundant networks provide the correct level of granularity if the focus is on conditional independencies, which is the case when studying moralisation and triangulation.
	\item \Cref{sec:categoriesnetworks} introduces a suitable notion of morphism for (irredundant) networks, culminating in the definition of categories $\bn$ and $\mn$ of Bayesian and Markov networks respectively.
	\item \Cref{sec:moralisation} characterises moralisation as a functor $\bn \to \mn$ and triangulation as a functor $\mn \to \bn$. Importantly, these functors have a high-level description, by pre-composition with the functors that define Bayesian and Markov networks. Moreover, as mentioned they are definable inductively on the string diagrammatic syntax representing the graph structure. These clauses, displayed as~\eqref{eq:moralisation_functor}-\eqref{eq:triangulation} below, encapsulate the `essence' of moralisation and triangulation, crystallising it into a syntactic presentation.  
\end{itemize}
Our work is just the beginning of a functorial framework for PGMs: we outline some research directions in \cref{sec:conclusions}. Missing proofs and details may be found in the appendices.

\section{Copy-Delete and Hypergraph Categories}\label{sec:cd-hypergraphcats}

The fundamental distinction between Bayesian and Markov networks is that one is a \emph{directed} and the other is an \emph{undirected} model. In this section we recall the categorical structures necessary to account for this difference. First, we focus on the directed case. Following the usual categorical perspective on Bayesian networks~\cite{fong2013causaltheoriescategoricalperspective,JacobsZ16}, we identify in \emph{copy-discard (CD) categories} the requirements needed to interpret the directed acyclic structures of these models. Intuitively, in CD-categories each object has a `copy' and a `delete' map, expressing the ability of nodes of being connected to multiple edges, or to no edge at all. We assume familiarity with \emph{string diagrams}~\cite{selinger11graphical,piedeleuzanasi}, the graphical language of monoidal categories, which we adopt to emphasise the interpretation of morphisms as graphical models.

\begin{assumption}[Strictness]\label{assump:strict}
	Throughout, we always mean monoidal categories and functors to be \newterm{strict}: associators, unitors and coherence morphisms for the functors are all identities.
\end{assumption}

\begin{definition}\label{def:cdcat}
	A \newterm{CD-category} is a symmetric monoidal category where each object $X$ is equipped with a commutative comonoid respecting the monoidal structure. We write `copy' (comultiplication) and `delete' (counit) maps in string diagram notation, respectively as $\input{copyX.tikz}$ and $\begin{tikzpicture}
	\begin{pgfonlayer}{nodelayer}
		\node [style=bn] (35) at (0.75, 0) {};
		\node [style=none] (36) at (-0.75, 0) {};
		\node [style=none] (37) at (-1.25, 0) {$X$};
	\end{pgfonlayer}
	\begin{pgfonlayer}{edgelayer}
		\draw (35) to (36.center);
	\end{pgfonlayer}
\end{tikzpicture}
$. We omit the object label when unnecessary. The commutative comonoid equations are then displayed as:
		\begin{equation}\label{eq:comonoids}
			\scalebox{0.7}{\input{copy_commutative.tikz}}\qquad \scalebox{0.7}{\input{copy_associative.tikz}}\qquad \scalebox{0.7}{\input{del_unit.tikz}}
		\end{equation}
	Associativity ensures a well-defined `copy' $\input{copymult.tikz}$ with multiple outputs.
	A \newterm{CD-functor} is a symmetric monoidal functor between CD-categories preserving copy and delete maps. CD-categories and CD-functors form a category $\cdcat$.
\end{definition}

\begin{example}\label{ex:finstoch}
	Our chief example of CD-category is $\finstoch$~\cite[Example 2.5]{fritz2019synthetic}, the category whose objects are finite sets\footnote{To ensure strictness of $\finstoch$, we actually take as objects only finite sets whose elements are finite lists, and define the monoidal product $X \otimes Y$ via list concatenation. For example, if $X=\lbrace [a],[b] \rbrace$ and $Y=\lbrace [a],[c] \rbrace$, then $X \otimes Y =\lbrace [a,a], [a,c], [b,a],[b,c] \rbrace$. This caveat is immaterial for our developments.
} and whose morphisms $f\colon X \to Y$ are maps $Y \times X \to \mathbb{R}^{\ge 0}$ such that $\Sigma_{y \in Y} f(y,x) = 1$. We will use a ``conditional notation'' and write $f(y\given x)$ for the image of $(y,x)$ via $f$. 
Note that the requirement on $f$ amounts to saying $f(-\given x)$ is a probability distribution on $Y$ for each $x \in X$. Another way to view $f$ is as a stochastic matrix (= a matrix where each column sums to $1$) with $X$-labelled columns and $Y$-labelled rows. Composition is defined via the Chapman--Kolmogorov equation, or equivalently by product of matrices: given $f \colon X \to Y$ and $g \colon Y \to Z$, $
				g\comp f (z \given x) := \sum_{y \in Y} g(z\given y) f(y \given x)$.
				The tensor product is the Kronecker product of matrices; more explicitly, for $f\colon X \to Y$ and $h \colon Z \to W$, $f \otimes h (y,w \given x,z) := f(y\given x) h(w \given z)$.
				The structural morphisms yielding the CD structure are defined as follows:
				\[
				\begin{array}{ccc}
					\input{copy.tikz} (y,z\given x) \coloneqq \begin{cases}
						1 & \text{if }x=y=z\\
						0 & \text{otherwise}
					\end{cases},&\text{and}&
				\begin{tikzpicture}
	\begin{pgfonlayer}{nodelayer}
		\node [style=bn] (35) at (0.75, 0) {};
		\node [style=none] (36) at (-0.75, 0) {};
	\end{pgfonlayer}
	\begin{pgfonlayer}{edgelayer}
		\draw (35) to (36.center);
	\end{pgfonlayer}
\end{tikzpicture}
 (\given x) \coloneqq 1.
				\end{array}
				\]
				This definition motivates the names `copy' and `delete' for the comonoid operations.
\end{example}

\begin{example}[Free CD-categories]\label{ex:freecdo}
Recall that a \emph{monoidal signature} is a pair $(A, \Sigma)$ consisting of a set $A$ of generating objects and a set $\Sigma$ of generating morphisms typed in $A^{\star}$ (finite lists of $A$-elements). One may freely construct the CD-category associated with $(A,\Sigma)$, denoted $\freecdo{A,\Sigma}$: its set of objects is $A^{\star}$ and morphisms are obtained by combining $\Sigma$-generators with $\input{copyX.tikz}$ and $\begin{tikzpicture}
	\begin{pgfonlayer}{nodelayer}
		\node [style=bn] (35) at (0.75, 0) {};
		\node [style=none] (36) at (-0.75, 0) {};
		\node [style=none] (37) at (-1.25, 0) {$X$};
	\end{pgfonlayer}
	\begin{pgfonlayer}{edgelayer}
		\draw (35) to (36.center);
	\end{pgfonlayer}
\end{tikzpicture}
$, for each $X \in A$, via sequential and parallel composition, and then quotienting by~\eqref{eq:comonoids} and the laws of symmetric strict monoidal categories. We refer e.g.\ to \cite{baez2018props,zanasi2018,bonchi2018deconstructing} for details.
\end{example}


We now turn attention to undirected models. We argue the suitable categorical structure is an extension of CD-categories, variously called \emph{hypergraph categories}~\cite{fong2019hypergraph} or \emph{well-supported compact closed categories}~\cite{CARBONI198711}.


\begin{definition}\label{def:hypcat}
	A \newterm{hypergraph category} is a CD-category where furthermore each object $X$ is equipped with a commutative monoid respecting the monoidal structure, and interacting with the comonoid on $X$ via the laws of special Frobenius algebras. We write `compare' (multiplication) and `omni' (unit) maps in string diagram notation, respectively as $\input{compareX.tikz}$ and $\begin{tikzpicture}
	\begin{pgfonlayer}{nodelayer}
		\node [style=bn] (24) at (-0.75, 0) {};
		\node [style=none] (25) at (0.75, 0) {};
		\node [style=none] (26) at (1.25, 0) {$X$};
	\end{pgfonlayer}
	\begin{pgfonlayer}{edgelayer}
		\draw (24) to (25.center);
	\end{pgfonlayer}
\end{tikzpicture}
$. The commutative monoid equations and special Frobenius equations are displayed as:
\begin{equation}\label{eq:hypergraph}
    \begin{aligned}
			\scalebox{0.7}{\input{compare_commutativity.tikz}}\qquad \scalebox{0.7}{\input{compare_associativity.tikz}}\qquad \scalebox{0.7}{\input{omni_unit.tikz}} \\[1em]
						\scalebox{0.7}{\input{frobenius_special.tikz}}\qquad \scalebox{0.7}{\input{frobenius_snake.tikz}} \qquad\qquad \qquad
   \end{aligned}
\end{equation}
Associativity ensures a well-defined `compare' $\input{comparemult.tikz}$ with multiple inputs.
 A \newterm{hypergraph functor} is a CD-functor preserving the monoid structure. Hypergraph categories and hypergraph functors form a category $\hypcat$.
\end{definition}

\begin{notation}
	We sometimes use the `cups' and `caps' notation: $\begin{tikzpicture}
	\begin{pgfonlayer}{nodelayer}
		\node [style=none] (37) at (-3.25, 0) {};
		\node [style=none] (38) at (-2, -0.75) {};
		\node [style=none] (39) at (-2, 0.75) {};
	\end{pgfonlayer}
	\begin{pgfonlayer}{edgelayer}
		\draw [in=-90, out=180] (38.center) to (37.center);
		\draw [in=180, out=90] (37.center) to (39.center);
	\end{pgfonlayer}
\end{tikzpicture}
\coloneqq \input{cup_def.tikz}$ and $\begin{tikzpicture}
	\begin{pgfonlayer}{nodelayer}
		\node [style=none] (29) at (-2, 0) {};
		\node [style=none] (30) at (-3.25, 0.75) {};
		\node [style=none] (31) at (-3.25, -0.75) {};
	\end{pgfonlayer}
	\begin{pgfonlayer}{edgelayer}
		\draw [in=90, out=0] (30.center) to (29.center);
		\draw [in=0, out=-90] (29.center) to (31.center);
	\end{pgfonlayer}
\end{tikzpicture}
\coloneqq \input{cap_def.tikz}$.
\end{notation}
	
The suitability of hypergraph categories to express undirected models may be better appreciated by observing that there is a bijective correspondence between homsets $[X,Y]$, $[I, X \otimes Y]$, and $[X \otimes Y,I]$ for any objects $X, Y$. These correspondences, obtained by `cups' and `caps', may be seen graphically as `bending wires', or `turning inputs into outputs' and vice versa: \input{fXY.tikz}\, \input{fcapXY.tikz} \, \input{fcupXY.tikz}.
We refer to~\cite{fong2019hypergraph} for a more systematic discussion. 
Additionally, hypergraph categories allow for the decomposition of `cups' and `caps' into the more elementary `copy', `delete', `compare', and `omni' maps.
The expressivity provided by these maps is actually crucial for modelling Bayesian and Markov networks, as all occurring variables are explicitly represented, ensuring that no information is hidden (cf.\ \cref{fig:ex_bayesian_network,fig:markov_network_ex}; see also \cite[Section 3]{jacobs2019causal_surgery}).
	
	\begin{example} \label{ex:main}
Our leading example of hypergraph category is $\mat$, the category whose objects are finite sets\footnote{The same caveat as in the previous footnote applies.} and whose morphisms $f\colon X \to Y$ are maps $Y \times X \to \mathbb{R}^{\ge 0}$. Sequential and parallel composition are defined as in $\finstoch$ (\Cref{ex:finstoch}), and we may use analogous conditional notation and matrix representation for its morphisms. In fact, we may regard $\mat$-morphisms as (generic) matrices with entries in $\mathbb{R}^{\ge 0}$. $\finstoch$ is the full subcategory where we restrict to the stochastic matrices. 
The compare-omni morphisms are defined as follows:
				\[
				\begin{array}{ccc}
					\input{compare.tikz} (z\given x,y) \coloneqq \begin{cases}
						1 & \text{if }x=y=z\\
						0 & \text{otherwise}
					\end{cases}, &\qquad& \begin{tikzpicture}
	\begin{pgfonlayer}{nodelayer}
		\node [style=bn] (24) at (-0.75, 0) {};
		\node [style=none] (25) at (0.75, 0) {};
	\end{pgfonlayer}
	\begin{pgfonlayer}{edgelayer}
		\draw (24) to (25.center);
	\end{pgfonlayer}
\end{tikzpicture}
 (x \given ) \coloneqq 1, 
				\end{array}
				\]
			Note that $\begin{tikzpicture}
	\begin{pgfonlayer}{nodelayer}
		\node [style=bn] (35) at (0.75, 0) {};
		\node [style=bn] (36) at (-0.75, 0) {};
	\end{pgfonlayer}
	\begin{pgfonlayer}{edgelayer}
		\draw (35) to (36);
	\end{pgfonlayer}
\end{tikzpicture}
 = \id_I$ is not necessarily satisfied.
	\end{example}

	\begin{example}\label{ex:finprojstoch}
		The category $\finprojstoch$, as defined in \cite[Definition 6.3]{stein2024exactconditions}, is a quotient of $\mat$ given by setting an equivalence relation $f_1 \propto f_2$ whenever there exists $\lambda\in \mathbb{R}^{> 0}$ such that $\lambda f_1(y\given x) = f_2(y\given x)$ for all $x \in X$ and $y \in Y$. Intuitively, $f_1$ and $f_2$ only differ by a global nonzero multiplicative factor.
		$\finprojstoch$ inherits the hypergraph category structure of $\mat$ via the functor $\mat \to \finprojstoch$ sending $f$ to its equivalence class $[f]_{\propto}$. 
		Also, $\finstoch$ embeds in $\finprojstoch$ via the analogously defined functor $f\mapsto [f]_{\propto}$.
	\end{example}
	
\begin{remark}[The normalisation cospan]\label{rem:normalisation_cospan}
	The introduction of $\finprojstoch$ is justified by the cospan 
	$\mat\xrightarrow{q} \finprojstoch \xleftarrow{i} \finstoch$, 
	where $q$ is the quotient and $i$ is the embedding described in \Cref{ex:finprojstoch}. This picture is central to our characterisation of Markov networks, as it describes a normalisation procedure (see~\cref{prop:irredundantmarkov_functor}). 
\end{remark}
	
	\begin{example}[Free hypergraph categories]\label{ex:freecor} Analogously to the case of CD-categories (\Cref{ex:freecdo}), one may freely construct a hypergraph category from a monoidal signature $(A,\Sigma)$: the only difference is that the construction of morphisms also involves `structural' generators $\input{compareX.tikz}$ and $\begin{tikzpicture}
	\begin{pgfonlayer}{nodelayer}
		\node [style=bn] (24) at (-0.75, 0) {};
		\node [style=none] (25) at (0.75, 0) {};
		\node [style=none] (26) at (1.25, 0) {$X$};
	\end{pgfonlayer}
	\begin{pgfonlayer}{edgelayer}
		\draw (24) to (25.center);
	\end{pgfonlayer}
\end{tikzpicture}
$ for each $X \in A$, and we additionally quotient by~\eqref{eq:hypergraph}.
	\end{example}

\section{Probabilistic Graphical Models as Functors}~\label{sec:networks}

In this section we characterise both Bayesian networks and Markov networks via the paradigm of functorial semantics. For Bayesian networks, we build on a previous result~\cite[Prop.~3.1]{jacobs2019causal_surgery}, whereas the characterisation for Markov networks is novel. These results also help us to characterise an `irredundant' version of networks, which is instrumental for the developments of next sections.

Throughout, all graphs are \newterm{finite}. Also, in a {DAG} (directed acyclic graph) $\g=(V_{\g},E_{\g})$, we denote by $\parents{v}$ the set of parents of a vertex $v$, i.e.~the set of vertices $w$ such that $w \to v$.

\begin{definition}\label{def:BN} A \newterm{Bayesian network} over $\g$ is given by an assignment $\tau(v)$ of a finite set for any vertex $v \in V_{\g}$ together with a stochastic matrix $\phi_v \colon \tau(v)\times \prod_{w \in\parents{v}} \tau(w)\to [0,1]$ interpreted as a conditional distribution of $v$ knowing $\parents{v}$.
Any such Bayesian network has an associated distribution given by $P(V_{\g})= \prod_{v \in V_{\g}} f_v (v \given \parents{v})$.
\end{definition}

\begin{example}\label{ex:BEAR}
The leftmost picture in \cref{fig:ex_bayesian_network} depicts a Bayesian network, borrowed from \cite[Ex. 9.14]{jacobs2021logical}. 
	Vertices represent an alarm ($A=\lbrace a,\no{a}\rbrace$), which may be activated by a burglary ($B=\lbrace b,\no{b}\rbrace$) or an earthquake ($E=\lbrace e, \no{e}\rbrace$). The radio ($R=\lbrace r,\no{r}\rbrace$) reliably reports any earthquake. The two elements of each set represent if the event occurred (e.g., $a$), or not (e.g., $\no{a}$). Edges indicate causal relationship between events, with probabilities given by stochastic matrices (represented as conditional probability tables).
	\begin{figure}[tp]
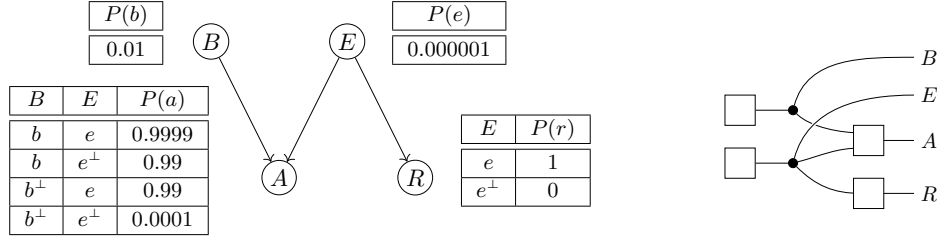

	\centering
	\begin{subfigure}[b]{0.6\textwidth}
		\centering
	\scalebox{0.9}{\input{BEAR.tikz}}
	\end{subfigure}
	\hspace*{0cm}
	\begin{subfigure}[b]{0.37\textwidth}
		\centering
	\scalebox{0.8}{\input{BEAR_stringdiagram.tikz}}
	\end{subfigure}
	\caption{A Bayesian network and the string diagram in $\finstoch$ representing it.}\label{fig:ex_bayesian_network}
	\end{figure}
\end{example}

\begin{remark}[On ordering]\label{rem:ordering}
	The stochastic matrix $f_v$ as in \Cref{def:BN} can be described by a morphism $\tau(v)\to \prod_{w \in\parents{v}} \tau(w)$ in $\finstoch$.
	However, this assignment is only unique up to permutation. It becomes unique once we choose a specific order for $\prod_{w \in\parents{v}} \tau(w)$. Thus for simplicity we work with (totally) \newterm{ordered graphs}. Recall that a DAG $\g$ is ordered if equipped with a topological ordering, i.e.\ a total order such that $v\to w$ implies $v <w$. 
\end{remark}
The idea behind the functorial perspective is to cleanly separate `syntax' (the graph) and `semantics' (the probability tables) of a Bayesian network. Viewing the graph as syntax is made possible by the free construction of \Cref{ex:freecdo}. Given an ordered DAG $\g=(V_{\g},E_{\g})$, we define $\cdsyn{\g}$ as the free CD-category given by the signature $(V_\mathcal{G}, \Sigma_{\mathcal{G}})$, where $\Sigma_{\mathcal{G}} \coloneqq \left\lbrace{\input{parents_morph.tikz} \given v \in V_{\mathcal{G}}}\right\rbrace$ and $\parents{v}$ indicates the parents of $v$. As stated in the following proposition, we can identify Bayesian networks based on $\g$ with models of $\g$ in $\finstoch$.

\begin{proposition}[{\cite[Proposition 3.1]{jacobs2019causal_surgery}}]\label{prop:bayesian_functor}
	Let $\g$ be an ordered DAG. 
	Bayesian networks over $\g$ are in bijective correspondence with CD-functors $\cdsyn{\g}\to \finstoch$.
\end{proposition}

A useful observation for our developments, not appearing in~\cite{jacobs2019causal_surgery}, is that $\cdsyn{\g}$ itself may be viewed as a functorial construction. Let $\odag$ be the category whose objects are ordered DAGs, and morphisms are order-preserving graph homomorphisms. Also, recall the category $\cdcat$ of CD-categories from \Cref{def:cdcat}.

\begin{theorem}\label{thm:graphhom-cd}
	There is a contravariant functor $\cdsyn{}\colon \odag  \to \cdcat$ mapping a DAG $\g$ to the CD-category $\cdsyn{\g}$. 
\end{theorem}
	The contravariance is motivated by considering the preimage of vertices. We provide some intuition on how the functor works via \cref{fig:graphhoms_bn_ex}, and defer the formal discussion to \cref{sec:graphhom_bn}.
 
\begin{figure}[tp]
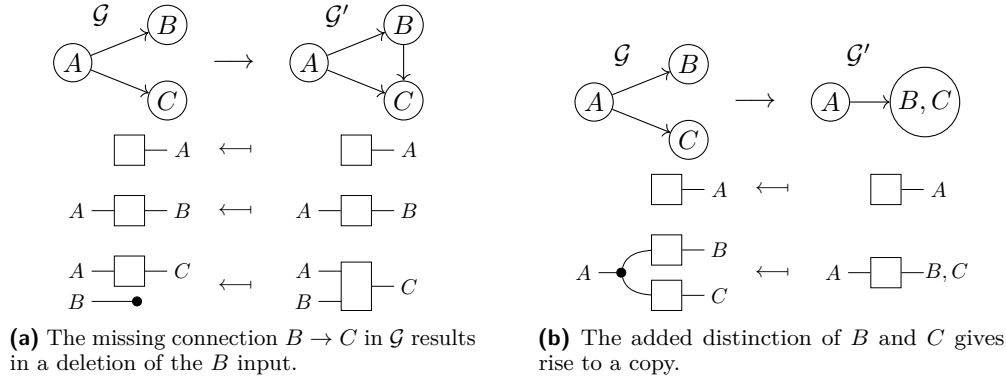

	\begin{subfigure}[b]{0.44\textwidth}
		\centering
		\input{odag_ex1.tikz}
		\par\vspace{1ex}
		\scalebox{0.8}{\input{odag_ex1_stringdiagram.tikz}}
	\subcaption{The missing connection $B\to C$ in $\g$ results in a deletion of the $B$ input.}\label{subfig:graphhoms_bn_ex1}
	\end{subfigure}
	\hspace*{0.6cm}
	\begin{subfigure}[b]{0.44\textwidth}
		\centering
		\input{odag_ex2.tikz}
		\par\vspace{1ex}
		\scalebox{0.8}{\input{odag_ex2_stringdiagram.tikz}}
	\subcaption{The added distinction of $B$ and $C$ gives rise to a copy.}\label{subfig:graphhoms_bn_ex2}
	\end{subfigure}
	\caption{Examples of the contravariant action of the functor of \Cref{thm:graphhom-cd} on morphisms $\g \to \g'$ (top), resulting in CD-functors (bottom), of which we describe the action on generators of $\cdsyn{\g}$.}\label{fig:graphhoms_bn_ex}
\end{figure}
	When the probability distribution associated to a Bayesian network does not have full support, i.e.\ there is some value $x$ such that $P(x)=0$, the Bayesian network structure has an inherent redundancy.
	For example, consider the family of Bayesian networks $F_p$, one for each $p \in [0,1]$,  described by 
	\[
	\input{dag_AB.tikz}
	\]
	where $A=\lbrace 0,1,2\rbrace$, $B=\lbrace b,\no{b} \rbrace$.
	All of these networks are associated to the same probability distribution $P(A,B)$, but they are different according to Definition~\ref{def:BN} (and \Cref{prop:bayesian_functor}). However, distinguishing them is not always desirable, for instance when studying the conditional independencies of $P(A,B)$. The important property is the existence, rather than the uniqueness of a Bayesian network representing the independencies of $P(A,B)$. 
		For our developments, it is important to account for this different perspective, which we do below.
			
\begin{definition}[Irredundant Bayesian Network] \label{def:irredundentBN}
	An \emph{irredundant} Bayesian network over $\g$ is a probability distribution $P$ admitting an assignment $\tau(v)$ of a finite set for any vertex $v \in V_{\g}$ such that $P(V_{\g})= \prod_{v \in V_{\g}} f_v (v \given \parents{v})$ for some stochastic matrices $f_v \colon \tau(v)\times \prod_{w \in\parents{v}} \tau(w)\to [0,1]$.
\end{definition}

We introduce a characterisation analogous to~\Cref{prop:bayesian_functor} for the irredundant case. Note $\bullet$, the one-vertex graph, is the final object in $\odag$. Thus given an ordered DAG $\g$, there always exists a unique morphism $\g \to \bullet$, contravariantly yielding a morphism (via \Cref{thm:graphhom-cd}) $!_{\g} \colon \cdsyn{\bullet}\to \cdsyn{\g}$. Another important observation is that, by definition, $\cdsyn{\bullet}$ is generated by the monoidal signature consisting of a single generating object $\bullet$ and a single generating morphisms \begin{tikzpicture}
	\begin{pgfonlayer}{nodelayer}
		\node [style=small box] (0) at (0, 0) {};
		\node [style=none] (1) at (2, 0) {};
	\end{pgfonlayer}
	\begin{pgfonlayer}{edgelayer}
		\draw (1.center) to (0);
	\end{pgfonlayer}
\end{tikzpicture}
, of type $I \to \bullet$. 
Therefore, any CD-functor $\cdsyn{\bullet}\to \finstoch$ is entirely captured by the assignments $\bullet \mapsto X$ and $\begin{tikzpicture}
	\begin{pgfonlayer}{nodelayer}
		\node [style=small box] (0) at (0, 0) {};
		\node [style=none] (1) at (2, 0) {};
	\end{pgfonlayer}
	\begin{pgfonlayer}{edgelayer}
		\draw (1.center) to (0);
	\end{pgfonlayer}
\end{tikzpicture}
 \mapsto \omega$, where $\omega \colon I \to X$ is a probability distribution on $X$. Paired with~\cref{prop:bayesian_functor}, this justifies the following characterisation for `distributions factoring through a Bayesian network'. 

\begin{proposition}\label{prop:irredundantbayesian_functor}
	Let $\g$ be an ordered DAG. 
	Irredundant Bayesian networks over $\g$ are in bijective correspondence with CD-functors $\omega \colon \cdsyn{\bullet}\to \finstoch$ factorising as $\cdsyn{\bullet}\xrightarrow{!_{\g}}\cdsyn{\g}\xrightarrow{F} \finstoch$ for some CD-functor $F$, i.e.\ such that $\omega = F !_{\g}$.
\end{proposition}

\begin{example}\label{ex:BEAR!}
	To better highlight the connection given by \cref{prop:irredundantbayesian_functor}, let us consider \cref{ex:BEAR}. 
	The CD-functor $!_{\g}\colon \cdsyn{\bullet}\to \cdsyn{\g}$ is then defined by sending $\bullet$ to $B\tensor E\tensor A\tensor R$ and $\begin{tikzpicture}
	\begin{pgfonlayer}{nodelayer}
		\node [style=small box] (0) at (0, 0) {};
		\node [style=none] (1) at (2, 0) {};
	\end{pgfonlayer}
	\begin{pgfonlayer}{edgelayer}
		\draw (1.center) to (0);
	\end{pgfonlayer}
\end{tikzpicture}
$ to the string diagram depicted in \cref{fig:ex_bayesian_network}. 
	When composing $!_{\g}$ with the CD-functor $\cdsyn{\g}\to \finstoch$ described by the probability tables in \cref{fig:ex_bayesian_network}, we indeed obtain the probability distribution associated to the Bayesian network.
\end{example}

We now focus on Markov networks, with the goal of achieving characterisations analogous to \Cref{prop:bayesian_functor,prop:irredundantbayesian_functor}. 
First, we recall how these models are defined in the literature. 

\begin{definition} 
	Given an undirected graph $\h=(V_{\h}, E_{\h})$, a \newterm{clique} $C$ is a subset of $V_{\h}$ such that for each $v,w\in C$, the set $\lbrace v,w \rbrace$ is an edge. 
The set of cliques is denoted by $\clique{\h}$.
A \newterm{Markov network} over $\h$ is given by an assignment $\tau(v)$ of a finite set for each vertex $v$, together with a function $\phi_C \colon \prod_{v \in C}\tau(v) \to \mathbb{R}^{\ge 0}$, called \emph{factor}, for each clique $C \in \clique{\h}$. 
\end{definition}
Such a Markov network yields a distribution defined by $P(V_{\h}) = \frac{1}{Z}\prod_{C\in \clique{\h}} \phi_C(C)$, where $Z$ is a normalisation coefficient. 
Note that $Z$ can be zero, thus $P$ is either a probability distribution or identically zero.
When $P$ satisfies the latter, we say that the Markov network is \newterm{degenerate}. The unnormalised distribution associated to a Markov network can be represented diagrammatically by comparing all occurrences of the outputs of the factors. Before formulating this construction in whole generality, we show it via an example.

\begin{figure}[!ht]
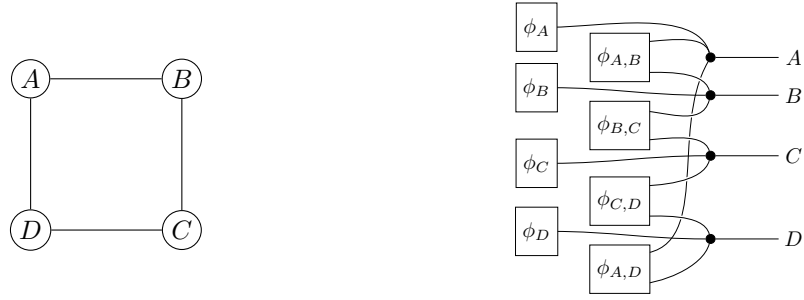

	\centering
	\begin{subfigure}[b]{0.43\textwidth}
		\centering
	\input{undirected_graph1.tikz}
	\end{subfigure}
	\hspace*{1cm}
	\begin{subfigure}[b]{0.43\textwidth}
		\centering
	\scalebox{0.8}{\input{markov_network_ex.tikz}}
	\end{subfigure}
	\caption{An undirected graph and the string diagram representing its unnormalised distribution.
	Although more graphically complex, the string diagram makes all contributing factors explicit. This approach is also commonly reflected in the theory of PGMs through the use of factor graphs, which are more descriptive.}
	\label{fig:markov_network_ex}
\end{figure}
\begin{example}\label{ex:misconception}
	The undirected graph in \cref{fig:markov_network_ex} illustrates the differences between Markov and Bayesian networks. In this scenario, adapted from \cite[Ex 3.8]{koller2009probabilistic}, four students --- Alice ($A$), Bob ($B$), Charles ($C$), and Debbie ($D$) --- meet in pairs to work on their class homework. $A$ and $C$ do not get along well, nor do $B$ and $D$, so the only pairs that do \emph{not} meet are these two.
	During class, the professor misspoke, leading to a potential misconception among the students.
	Since $A$ and $C$ do not communicate directly to each other, they can only influence each other through $B$ and $D$. 
	Similarly, $B$ and $D$ only influence each other through $A$ and $C$.
	Note the symmetry of these relationships cannot be captured by a Bayesian network, as it inherently requires a specific choice of directionality among student-vertices. An example of a Markov network over this graph is given by the factors below, where vertices are associated with sets $A=\lbrace a,\no{a}\rbrace$, $B=\lbrace b,\no{b}\rbrace$, $C=\lbrace c,\no{c}\rbrace$, $D=\lbrace d,\no{d}\rbrace$.
	\begin{equation}\label{tab:misconception_factors}
		{\begin{tabular}{|c|c|c|}
		\hline
		\multicolumn{3}{|c|}{$\phi_{AB}$}\\
		\hline\hline
		$a$&$b$& $10$ \\
		\hline 
		$a$ & $\no{b}$ & $1$\\
		\hline 
		$\no{a}$& $b$ & $5$\\
		\hline
		$\no{a}$&$\no{b}$ & $30$\\
		\hline
		\end{tabular}}	
		\hspace{0.3cm}
		{\begin{tabular}{|c|c|c|}
		\hline
		\multicolumn{3}{|c|}{$\phi_{BC}$}\\
		\hline\hline
		$b$&$c$& $100$ \\
		\hline 
		$b$ & $\no{c}$ & $1$\\
		\hline 
		$\no{b}$& $c$ & $1$\\
		\hline
		$\no{b}$&$\no{c}$ & $100$\\
		\hline
		\end{tabular}}	
		\hspace{0.3cm}
		{\begin{tabular}{|c|c|c|}
		\hline
		\multicolumn{3}{|c|}{$\phi_{CD}$}\\
		\hline\hline
		$c$&$d$& $1$ \\
		\hline 
		$c$ & $\no{d}$ & $100$\\
		\hline 
		$\no{c}$& $d$ & $100$\\
		\hline
		$\no{c}$&$\no{d}$ & $1$\\
		\hline
		\end{tabular}}	
		\hspace{0.3cm}
		{\begin{tabular}{|c|c|c|}
		\hline
		\multicolumn{3}{|c|}{$\phi_{AD}$}\\
		\hline\hline
		$a$&$d$& $100$ \\
		\hline 
		$a$ & $\no{d}$ & $1$\\
		\hline 
		$\no{a}$& $d$ & $1$\\
		\hline
		$\no{a}$&$\no{d}$ & $100$\\
		\hline
		\end{tabular}}
	\end{equation} 

	In our interpretation, $\no{x}$ indicates that $X$ does not have the misconception. As we do not wish to consider an inherent difference between the various students, all the factors over a single node are omitted (we can simply set them to be constantly one). 
	The chosen numbers highlight some key aspects of the students relationships: $C$ and $D$ are prone to disagree, while all other pairs tend to agree. Additionally, $A$ and $B$ are less likely to have the misconception, and the fact that $\phi_{AB}(a,\no{b})< \phi_{AB} (\no{a},b)$ indicates that when they do disagree, it is more plausible that $B$ had the misconception.
\end{example}

To establish an analogue of \cref{prop:bayesian_functor} for Markov networks, we use the free construction of \cref{ex:freecor}.
Given any \emph{ordered} undirected graph $\h=(V_\mathcal{H}, E_{\mathcal{H}})$, consider the signature given by $(V_\mathcal{H}, \Sigma_{\mathcal{H}})$, with $\Sigma_{\mathcal{H}} \coloneqq \left\lbrace{\begin{tikzpicture}
	\begin{pgfonlayer}{nodelayer}
		\node [style=small box] (0) at (0, 0) {};
		\node [style=none] (1) at (1.5, 0) {};
		\node [style=none] (2) at (2, 0) {$C$};
	\end{pgfonlayer}
	\begin{pgfonlayer}{edgelayer}
		\draw (1.center) to (0);
	\end{pgfonlayer}
\end{tikzpicture}
 \given C \in \clique{\mathcal{H}}}\right\rbrace$, and define $\syn{\mathcal{H}} \coloneqq \freehyp{V_\mathcal{H}, \Sigma_{\mathcal{H}}}$.
Note that the order here is necessary to give a well-defined box for each clique, since otherwise we would not know how to order the outputs (cf.~\cref{rem:ordering}).
\begin{proposition}\label{prop:markov_functor}
	Let $\h$ be an ordered undirected graph.
	Markov networks over $\h$ are in bijective correspondence to hypergraph functors $\syn{\h} \to \mat$.
\end{proposition}

For instance, the Markov network of \Cref{ex:misconception} yields $V_\mathcal{H} = \{A,B,C,D\}$ and $\Sigma_{\mathcal{H}}$ including all the generators $\phi$ appearing in the string diagram of \cref{fig:markov_network_ex}, and the functor $\syn{\h}\to \mat$ is defined according to the factors defined in \eqref{tab:misconception_factors}.

Just as in the case of Bayesian networks, also the construction of $\syn{\h}$ yields a functor. The categories involved are $\ougr$, the category of ordered undirected graphs and order-preserving graph homomorphisms, and $\hypcat$, introduced in \Cref{def:hypcat}.

\begin{theorem}\label{thm:graphhom-hyp}
	There is a contravariant functor $\syn{}\colon \ougr  \to \hypcat$ mapping a ordered undirected graph $\h$ to the hypergraph category $\syn{\h}$. 
\end{theorem}
We provide some intuition on how the functor works via \cref{fig:graphhoms_mn_ex}, and defer the formal discussion to \cref{sec:graphhom_mn}.

	The observation about redundancy (\cref{def:irredundentBN}) is even more relevant for Markov networks. The generality of cliques causes a considerable redundancy of information, which is often `pulled under the rug' in applications, but should be made explicit in a mathematically rigorous treatment. 
	
\begin{definition}[Irredundant Markov Network]\label{def:irredundantMN}
	An \emph{irredundant} Markov network over an (ordered) undirected graph $\h$ is a probability distribution $P$ admitting an assignment $\tau(v)$ of a finite set for any vertex $v\in V_{\h}$ such that $P(V_{\h})=\frac{1}{Z}\prod_{C\in \clique{\h}} \phi_C(x_C)$ for some factors $\phi_C$, where $Z$ is a normalisation coefficient.
\end{definition}

Note that, because $P$ is a proper distribution, $Z \neq 0$, so any irredundant network is also non-degenerate. As in the case of Bayesian networks, $\bullet$ is the final object in $\ougr$, yielding a morphism $!_{\h}\colon \syn{\bullet}\to \syn{\h}$. However, to take care of the normalisation coefficient, a characterisation analogous to the one of \Cref{prop:irredundantbayesian_functor} requires a few extra pieces. 
	
	\begin{proposition}\label{prop:irredundantmarkov_functor}
	Let $\h$ be an ordered undirected graph. 
	Irredundant Markov networks over $\h$ are in bijective correspondence with CD-functors $\omega \colon \cdsyn{\bullet}\to \finstoch$ for which there exists a commutative diagram as follows for some hypergraph functor $\syn{\h} \to \mat$:
	\begin{equation}\label{eq:mn_factorisation}
		\begin{tikzcd}[row sep=0.1pt]
			& \syn{\h} \ar[r] &\mat\ar[dr,"q"]&  \\
			\syn{\bullet} \ar[ru,"!"] &&& \finprojstoch\\
			& \cdsyn{\bullet} \ar[lu,"\star"]\ar[r,"\omega"] & \finstoch \ar[ur,"i"]&
		\end{tikzcd}
	\end{equation}
	where $\star \colon \cdsyn{\bullet}\to \syn{\bullet}$ sends the unique generator to itself, and $q,i$ are the functors of the normalization cospan, see \cref{rem:normalisation_cospan}.
	\end{proposition}
The commutativity of~\eqref{eq:mn_factorisation} states precisely that the probability distribution $\omega$ associated with $\omega \colon \cdsyn{\bullet}\to \finstoch$ and the unnormalised distribution $Q$ associated with $\cdsyn{\bullet}\to \mat$ satisfies $Q \propto \omega$, i.e.\ there exists a normalisation coefficient $Z$ such that $\omega = \frac{1}{Z}Q$.
\begin{example}
To better understand \cref{prop:irredundantmarkov_functor}, let us consider \cref{ex:misconception}.
The hypergraph functor $!_{\h}\colon \syn{\bullet}\to \syn{\h}$ sends $\bullet$ to $A\tensor B \tensor C \tensor D$ and $\begin{tikzpicture}
	\begin{pgfonlayer}{nodelayer}
		\node [style=small box] (0) at (0, 0) {};
		\node [style=none] (1) at (2, 0) {};
	\end{pgfonlayer}
	\begin{pgfonlayer}{edgelayer}
		\draw (1.center) to (0);
	\end{pgfonlayer}
\end{tikzpicture}
$ to the string diagram in \cref{fig:markov_network_ex}.
The unnormalised distribution $Q$ is then described by the composition $\syn{\bullet}\xrightarrow{!} \syn{\h}\to \mat$, where $\syn{\h}\to \mat$ is determined by \eqref{tab:misconception_factors}.
\end{example}

\begin{figure}[tp]
	\begin{subfigure}[b]{0.44\textwidth}
		\centering
		\input{ougr_ex1.tikz}
		\par\vspace{1ex}
		\scalebox{0.8}{\input{ougr_ex1_stringdiagram.tikz}}
	\subcaption{The missing connection $B\edge C$ in $\h$ results in disregarding the clique $\lbrace B,C\rbrace$ and consequently also $\lbrace A,B,C\rbrace$. For brevity, we used $X$ as a placeholder for the other cliques $\lbrace A\rbrace$, $\lbrace B\rbrace$, $\lbrace C\rbrace$, $\lbrace A,B\rbrace$, and $\lbrace A,C\rbrace$.
	}\label{subfig:graphhoms_mn_ex1}
	\end{subfigure}
	\hspace*{0.6cm}
	\begin{subfigure}[b]{0.44\textwidth}
		\centering
		\input{ougr_ex2.tikz}
		\par\vspace{1ex}
		\scalebox{0.8}{\input{ougr_ex2_stringdiagram.tikz}}
	\subcaption{The added distinction of $B$ and $C$ gives rise to a compare.}\label{subfig:graphhoms_mn_ex2}
	\end{subfigure}
\caption{Examples of the contravariant action of the functor of \Cref{thm:graphhom-hyp} on morphisms $\h \to \h'$ (top), resulting in hypergraph functors (bottom), of which we describe the action on generators of $\syn{\h'}$.}\label{fig:graphhoms_mn_ex}
\end{figure}


\section{Morphisms Between Networks}\label{sec:categoriesnetworks}

In order to define functorial transformations between Bayesian and Markov networks, we need a full definition of the categories involved. 
The characterisations of \Cref{sec:networks} only provide the objects of these categories. 
We now identify suitable notions of morphism, culminating in definitions of the categories of Bayesian networks and Markov networks (\cref{def:bn_cat,def:mn_cat}). 
As our ultimate aim is to describe moralisation and triangulation, and these modifications pertain specifically to the study of conditional independencies, the most natural approach is to focus attention to irredundant Bayesian and Markov networks. 

Given the functorial perspective on networks, a natural candidate for morphisms are compatible pairs of a `morphism between syntaxes' and a `morphism between semantics'. The former will simply be an order-preserving graph homomorphism. For the latter, recall that both irredundant Bayesian and Markov networks are defined by probability distributions factoring through a certain structure. In $\finstoch$ we may regard two such distributions as maps $\omega\colon I \to X$ and $\omega'\colon I \to Y$, and a morphism between them as a stochastic matrix $f\colon X \to Y$ such that $\omega' = f\comp \omega$. Now, in the functorial perspective $\omega,\omega'$ are identified with CD-functors $\cdsyn{\bullet} \to \finstoch$. We can lift the notion of morphism between distributions from $\finstoch$ to the level of such functors, as follows.
\begin{definition}
	Let $\omega,\omega' \colon \cdsyn{\bullet} \to \finstoch$ be two CD-functors.
	A \newterm{monoidal transformation} $\eta \colon \omega \to \omega'$ is a family of morphisms $\eta_X\colon \omega(X) \to \omega'(X)$ for every object $X \in \cdsyn{\bullet}$ such that $\eta_{X\otimes Y}=\eta_X \otimes \eta_Y$ and $\eta_I =\id_I$.
	We say that a monoidal transformation is \newterm{distribution-preserving} if $\eta_{\bullet} \omega(\begin{tikzpicture}
	\begin{pgfonlayer}{nodelayer}
		\node [style=small box] (0) at (0, 0) {};
		\node [style=none] (1) at (2, 0) {};
	\end{pgfonlayer}
	\begin{pgfonlayer}{edgelayer}
		\draw (1.center) to (0);
	\end{pgfonlayer}
\end{tikzpicture}
) = \omega' (\begin{tikzpicture}
	\begin{pgfonlayer}{nodelayer}
		\node [style=small box] (0) at (0, 0) {};
		\node [style=none] (1) at (2, 0) {};
	\end{pgfonlayer}
	\begin{pgfonlayer}{edgelayer}
		\draw (1.center) to (0);
	\end{pgfonlayer}
\end{tikzpicture}
)$. 
\end{definition}
The resulting notion is \emph{not} a natural transformation, as one may initially expect: the naturality requirement is too strong, as it only holds when $f$ is a deterministic function (each column has exactly one non-zero value). To justify our notion, observe that distribution-preserving monoidal transformations $\eta\colon \omega \to \omega'$ are in bijective correspondence with stochastic matrices $f$ satisfying $\omega'(\begin{tikzpicture}
	\begin{pgfonlayer}{nodelayer}
		\node [style=small box] (0) at (0, 0) {};
		\node [style=none] (1) at (2, 0) {};
	\end{pgfonlayer}
	\begin{pgfonlayer}{edgelayer}
		\draw (1.center) to (0);
	\end{pgfonlayer}
\end{tikzpicture}
) = f\comp \omega(\begin{tikzpicture}
	\begin{pgfonlayer}{nodelayer}
		\node [style=small box] (0) at (0, 0) {};
		\node [style=none] (1) at (2, 0) {};
	\end{pgfonlayer}
	\begin{pgfonlayer}{edgelayer}
		\draw (1.center) to (0);
	\end{pgfonlayer}
\end{tikzpicture}
)$. 
Indeed, the objects of $\cdsyn{\bullet}$ are of the form $\bullet^{\tensor n}$ for some $n$, and thus $\eta_{\bullet^{\otimes n}}=f^{\otimes n}$ describes such a bijection. We are ready to define the two categories of networks.

\begin{definition}\label{def:bn_cat}
	The \newterm{category of (irredundant) Bayesian networks} $\bn$ is defined as:
\begin{itemize}
	\item The objects are pairs $(\omega , \g)$, where $\omega$ is a CD-functor $\cdsyn{\bullet} \to \finstoch$ factoring through the CD-functor $!_{\g} \colon \cdsyn{\bullet}\to \cdsyn{\g}$ associated to $\g \to \bullet$. 
	\item A morphism $(\omega,\g)\to (\omega',\g')$ is a pair $(\alpha,\eta)$, where $\alpha$ is a morphism $\alpha\colon \g' \to \g$  in $\odag$ and $\eta$ is a distribution-preserving monoidal transformation $\omega \to \omega'$. Composition is component-wise.
\end{itemize}
\end{definition}
\begin{definition}\label{def:mn_cat}
	The \newterm{category of (irredundant) Markov networks} $\mn$ is defined as:
	\begin{itemize}
		\item The objects are pairs $(\omega , \h)$, where $\omega$ is a CD-functor $\cdsyn{\bullet} \to \finstoch$ factoring through $\syn{\h}$ according to~\eqref{eq:mn_factorisation}. 
		\item A morphism $(\omega,\h)\to (\omega',\h')$ is given by a morphism $\alpha \colon \h' \to \h$ in $\ougr$ together with a distribution-preserving monoidal transformation $\eta\colon \omega \to \omega'$.  Composition is component-wise.
	\end{itemize}	
\end{definition}

The choice of objects is justified by~\cref{prop:irredundantbayesian_functor,prop:irredundantmarkov_functor}. The chosen directionality of morphisms captures the process of \emph{revealing information}. 
To support this claim, note that for any distribution $\omega$, a morphism $(\omega,\bullet)\to (\omega,\g)$ provides more information about $\omega$ in the form of a Bayesian network structure.
Further insight is gained by considering extra variables.  
Explicitly, take $(\omega,\g)$ and some additional $v\notin V_{\g}$. 
Then, any stochastic matrix $f$ with domain a subset of $V_{\g}$ and codomain $v$ can be used to update the distribution: set $\begin{tikzpicture}
	\begin{pgfonlayer}{nodelayer}
		\node [style=none] (65) at (1.25, 0) {};
		\node [style=small box] (66) at (-1, 0) {$\omega'$};
	\end{pgfonlayer}
	\begin{pgfonlayer}{edgelayer}
		\draw (65.center) to (66);
	\end{pgfonlayer}
\end{tikzpicture}
\coloneqq \input{fcopyw.tikz}$. 
The natural graph $\g'$ associated with $\omega'$ is obtained from $\g$ by adding the vertex $v$ and edges $w\to v$ whenever $w$ is an input of $f$.
This defines a revealing variable morphism $(\omega,\g) \to (\omega',\g')$, given by the surjection $\g' \to \g$ together with $f$.

\begin{remark}
	The chosen direction is not suited to \emph{hide} information. 
	For example, take $(\omega,\g)$ with $\g$ given by $\begin{tikzpicture}
	\begin{pgfonlayer}{nodelayer}
		\node [style=bw] (15) at (0, 0) {$A$};
		\node [style=bw] (29) at (2.5, 0) {$C$};
		\node [style=bw] (31) at (-2.5, 0) {$B$};
	\end{pgfonlayer}
	\begin{pgfonlayer}{edgelayer}
		\draw [style=arrow] (15) to (29);
		\draw [style=arrow] (15) to (31);
	\end{pgfonlayer}
\end{tikzpicture}
$ and imagine we want to delete $A$. 
	In general, the marginal $\omega'$ will display some dependence between $B$ and $C$, so we need to consider $\g'$ on $B$ and $C$ with an edge in either direction. 
	But then the inclusion $\g'\to \g$ is not a graph homomorphism, so this hiding process cannot be seen as a morphism $(\omega,\g)\to (\omega',\g')$.
	Instead, the surjection $\g \to \g'$ sending $A$ to either $B$ or $C$, together with the conditional of $A$ given $B$ and $C$, gives a morphism $(\omega',\g')\to (\omega,\g)$.
\end{remark}

\section{Moralisation and Triangulation as Functors}\label{sec:moralisation}

Moralisation and triangulation are well-known transformations that bridge Bayesian and Markov networks. 
Their importance lies in enabling reasonings and inference methods that are specific to the other type of network. 
A key example is the junction tree algorithm of Markov networks, which is widely used in machine learning to extract information about marginals and conditionals, see e.g.\ \cite{koller2009probabilistic,barberBRML2012}.

The payoff of our perspective is the ability to describe such network transformations by only manipulating the syntactic level, i.e. the categories $\cdsyn{\g}$ and $\syn{\h}$. Moralisation and triangulation will amount to precomposition with inductively defined functors, see \eqref{eq:moralisation_functor}, \eqref{eq:triangulation} below. We begin by focusing on moralisation, which transforms a Bayesian network $\g$ into a Markov network $\mor{\mathcal{G}}$, with the property that a distribution $\omega$ factorising through $\g$ also factorises through $\mor{\mathcal{G}}$~\cite{koller2009probabilistic}.

\begin{definition}\label{def:moralisation}
Let $\g$ be a DAG. Its \newterm{moralisation} $\mor{\mathcal{G}}$ is the undirected graph whose vertices are the same as $\mathcal{G}$ and there is an edge between $v$ and $w$ (written $v \edge w$) whenever  in $\g$ there is an edge between them, or they are both parents of the same vertex.
\end{definition}

We now discuss how to turn moralisation into a functor $\bn \to \mn$ (see \cref{sec:functoriality_proofs} for the details). Recall that an (irredundant) Bayesian network, that is, an object of $\bn$, is a pair $(\omega, \g)$ of a DAG $\g$ and a functor $\omega \colon \cdsyn{\bullet} \to \finstoch$ factorising as
$\cdsyn{\bullet}\xrightarrow{!_{\g}} \cdsyn{\g}\xrightarrow{F} \finstoch$ for some $F \colon \cdsyn{\g} \to \finstoch$. 
	The corresponding moralisation, an object of $\mn$, is going to be defined as $(\omega,\mor{\g})$, where $\omega$ should factorise through $\syn{\mor{\g}}$ according to \eqref{eq:mn_factorisation}, instantiated as follows.
	\begin{equation}\label{eq:factoringdistribution_maintext}
		\begin{tikzcd}[row sep=0.1pt]
			& \syn{\mor{\g}} \ar[r] &\mat\ar[dr,"q"]&  \\
			\syn{\bullet} \ar[ru,"!"] &&& \finprojstoch\\
			& \cdsyn{\bullet} \ar[lu,"\star"]\ar[r,"\omega"] & \finstoch \ar[ur,"i"]&
		\end{tikzcd}
	\end{equation}
	The outstanding question is how to correctly define $\syn{\mor{\g}} \to \mat$ in~\eqref{eq:factoringdistribution_maintext}. This goes in three steps. First, we may construct $\syn{\g} \coloneqq \freehyp{V_{\g}, \Sigma_{\g}}$, the free \emph{hypergraph} category on $\g$, which comes with an induced $\star_{\g}\colon \cdsyn{\g}\to \syn{\g}$ given by the identity on the generators $\Sigma_{\g}$.
Intuitively, this amounts to taking an ``indirect'' perspective on the data of $\g$. Second, observe that $F \colon \cdsyn{\g} \to \finstoch$ also yields a hypergraph functor $\tilde{F} \colon \syn{\g} \to \mat$, defined on generators the same way as $F$. Third, we define $m \colon \syn{\mor{\g}} \to \syn{\g}$ as the hypergraph functor freely obtained by the following mapping on the generators of $\syn{\mor{\g}}$:
\begin{equation}\label{eq:moralisation_functor}
	\scalebox{0.8}{\begin{tikzpicture}
	\begin{pgfonlayer}{nodelayer}
		\node [style=small box] (0) at (0, 0) {};
		\node [style=none] (1) at (1.5, 0) {};
		\node [style=none] (2) at (2, 0) {$C$};
	\end{pgfonlayer}
	\begin{pgfonlayer}{edgelayer}
		\draw (1.center) to (0);
	\end{pgfonlayer}
\end{tikzpicture}
} \qquad \longmapsto\qquad \begin{cases}
		\scalebox{0.8}{\input{graph_parents_morph.tikz}} & \text{if }C=\lbrace v \rbrace \cup \parents{v} \text{ for some }v \in \g\\
		\scalebox{0.8}{\input{omni_clique.tikz}} & \text{otherwise}
	\end{cases}
\end{equation}
This simple description mimics at the level of string diagrammatic syntax the transformation described by \cref{def:moralisation}. Putting these all together, we obtain the desired hypergraph functor $\syn{\mor{\g}}\xrightarrow{m} \syn{\g} \xrightarrow{\tilde{F}} \mat$, thus completing the definition of the Markov network $(\omega, \mor{\g})$ given by commutativity of~\eqref{eq:factoringdistribution_maintext}. More explicitly, this amounts to the following commutative diagram
\begin{equation}
	\begin{tikzcd}
		\syn{\mor{\g}}\ar[r,"m"]& \syn{\g} \ar[r,"\tilde{F}"] &\mat\ar[dr,"q"]&  \\
		\syn{\bullet} \ar[u,"!_{\mor{\g}}"]\ar[ru,"\tilde{!}_{\g}"] &\cdsyn{\g}\ar[u,"\star_{\g}"]\ar[dr,"F"] && \finprojstoch\\
		& \cdsyn{\bullet} \ar[lu,"\star_{\bullet}"]\ar[r,"\omega"]\ar[u,"!_{\g}"] & \finstoch\ar[uu,hook]\ar[ur,"i"]&
	\end{tikzcd}
\end{equation}
Observe that the only non-obvious step in this process is the definition of $m \colon \syn{\mor{\g}} \to \syn{\g}$. This is the key piece of our functorial view of moralisation, providing two insights: first, moralisation may be decomposed into an inductive definition on the string diagrammatic syntax capturing the graph structures; second, by regarding networks themselves as functors, moralisation may be simply defined by functor precomposition.

\begin{example}
	The moralisation of the DAG in \cref{ex:BEAR} is given by adding an additional edge between $B$ and $E$, since they are both parents of $A$. $m \colon \syn{\mor{\g}} \to \syn{\g}$ maps the string diagram in $\syn{\mor{\g}}$ representing the moralised network, below left, to the string diagram representing the Bayesian network, below right:
	\begin{equation*}
		\scalebox{0.75}{\input{BEAR_moral2.tikz}}
	\end{equation*}
	where we used $\phi$ to denote the generators in $\syn{\mor{\g}}$ and $f$ for the generators in $\syn{\g}$.
	Postcomposing $m$ with $\syn{\g}\to \mat$ sets $\phi_{BA}$, $\phi_{EA}$, $\phi_B$, $\phi_E$ and $\phi_{ER}$ as the stochastic matrices of the Bayesian network, while the other factors are all set to be constantly one.
\end{example}

In full generality, we are able to conclude the following (see \cref{sec:functoriality_proofs} for a proof).
\begin{theorem}\label{thm:moralisation}
	Moralisation gives rise to a functor $\mor{-}\colon \bn \to \mn$ which on objects maps ${(\omega,\g)}$ to $(\omega,\mor{\g})$.
\end{theorem}

Conversely, in the theory of probabilistic graphical models, obtaining a Bayesian network from a Markov network involves a process known as \emph{triangulation}.

\begin{definition}
	Let $\h$ be an ordered undirected graph. Then its \newterm{triangulation} $\tr{\h}$ is the ordered graph where $v\to w$ whenever $v\le w$ and there is a path $v \edge w_1 \edge \dots \edge w_n =w$, with $w_i \ge w$ for each $i=1,\dots,n$.
\end{definition}
In particular, every edge $v\edge w$ in $\h$ yields a directed edge in $\tr{\h}$. The idea is that $\tr{\h}$ only enforces the conditional independencies that we know to hold thanks to $\h$.

As in the case of the moralisation, we want to capture triangulation as a functor, but in the opposite direction $\mn \to \bn$. The construction is similar: the key step is defining the functor $t\colon \syn{\tr{\h}} \to \syn{\h}$ which will act by precomposition on the given syntax category $\syn{\h}$ of the Markov network $\h$. It is defined on the generators of $\syn{\h}$ by the following clause
\begin{equation}\label{eq:triangulation}
	\scalebox{0.8}{\input{parents_morph.tikz}}\qquad \longmapsto \qquad \scalebox{0.8}{\input{triangulation_def.tikz}}
\end{equation}
where $\compcomp{C_v}$ is a \emph{compare-composition} (see Appendix~\ref{sec:graphhom_mn}), determined by comparing all the outputs of the set of morphisms $C_v\coloneqq \lbrace \phi \colon I\to C \mid v\in C\subseteq \parents{v} \cup \lbrace v \rbrace\rbrace \subseteq \Sigma_{\h}$, $P\coloneqq \parents{v} \cap C_v$ and $Q\coloneqq \parents{v} \setminus P$.
Intuitively, the right-hand side compares all the factors involving $v$ and its parents, while $Q$ accounts for the remaining inputs.
By showing that $\syn{\bullet}\to \syn{\h}$ factors through $\syn{\tr{\h}}$, it seems we have everything to conclude, since we have the following commutative diagram.
\[
\begin{tikzcd}[row sep=1pt,column sep=large]
	\cdsyn{\bullet}\ar[r,"!_{\tr{\h}}"]\ar[dd,"\omega"]&\cdsyn{\tr{\h}} \ar[ddl,dashed]\ar[r,"\star_{\tr{\h}}"]& \syn{\tr{\h}}\ar[dr,"t"]\ar[dd]&\\
	&&& \syn{\h}\ar[dl]\\
	\finstoch\ar[r,"i"] & \finprojstoch & \mat\ar[l,"q" above]&
\end{tikzcd}
\]
However, we need to ensure that $\syn{\tr{\h}} \to \mat$ yields the functor $\cdsyn{\tr{\h}}\to \finstoch$, dashed above, that preserves the associated distribution $\cdsyn{\bullet}\to \finstoch$. This is true for $\tr{\h}$, due to the peculiar structure of triangulation {(see \cref{sec:functoriality_proofs}, \cref{lem:triangulation_finstoch})}, which allows to conclude \cref{thm:triangulation} below. Interestingly, it is not necessarily true for generic DAGs, see \cref{rem:v-structures} below.

\begin{theorem}\label{thm:triangulation}
	Triangulation gives rise to a functor $\tr{-}\colon \mn \to \bn$ which on objects maps $(\omega,\h)$ to $(\omega,\tr{\h})$.
\end{theorem}


\begin{remark}\label{rem:v-structures}
	Consider the the DAG $\g$ given by $\begin{tikzpicture}
	\begin{pgfonlayer}{nodelayer}
		\node [style=bw] (15) at (0, 0) {$A$};
		\node [style=bw] (29) at (2.5, 0) {$C$};
		\node [style=bw] (31) at (5, 0) {$B$};
	\end{pgfonlayer}
	\begin{pgfonlayer}{edgelayer}
		\draw [style=arrow] (15) to (29);
		\draw [style=arrow] (31) to (29);
	\end{pgfonlayer}
\end{tikzpicture}
$.
	Any Bayesian network over $\g$ makes $A$ and $B$ independent of each other once we discard $C$.
	We claim that this independence is not necessarily true for a hypergraph functor $\syn{\g}\to \mat$.
	To this end, consider $A=B=C=\lbrace 0,1 \rbrace$, and let $\phi\colon I \to A\tensor B\tensor C$ be the morphism defined by $\phi(a,b,c)=1$ if exactly two of them are equal and 0 otherwise.
	We then define $\Phi \colon \syn{\g}\to \mat$ by sending the generators $I\to A$ and $I\to B$ to $\begin{tikzpicture}
	\begin{pgfonlayer}{nodelayer}
		\node [style=bn] (24) at (-0.75, 0) {};
		\node [style=none] (25) at (0.75, 0) {};
	\end{pgfonlayer}
	\begin{pgfonlayer}{edgelayer}
		\draw (24) to (25.center);
	\end{pgfonlayer}
\end{tikzpicture}
$ and $A\tensor B \to C$ to \input{phiC.AB.tikz}. 
	In this way, $\phi$ corresponds to the distribution $\syn{\bullet}\xrightarrow{!} \syn{\g}\xrightarrow{\Phi} \mat$.
	By direct computations, $\input{phi_noC.tikz}\neq \input{indep.tikz}$, so indeed $A$ and $B$ share some dependence even when $C$ is discarded.
\end{remark}

\begin{example}\label{ex:misconception_triangulation}
The triangulation of the undirected graph $\h$ of \cref{ex:misconception} is the DAG obtained by making all edges direct and adding an edge $A\to C$ (because of the path $A\edge D \edge C$). The string diagram representing $\tr{\h}$ in $\syn{\tr{\h}}$ is given below left, with its image under the functor $t\colon \syn{\tr{\h}} \to \syn{\h}$ below right: 
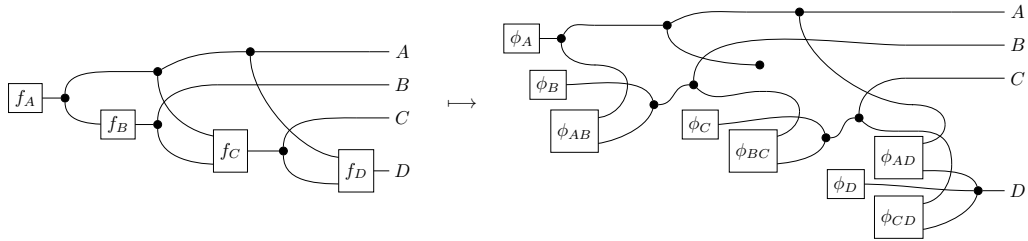
\begin{equation*}
	\scalebox{0.7}{\input{misconception_triangulation.tikz}}
\end{equation*}
The equations in \eqref{eq:comonoids} and \eqref{eq:hypergraph} ensure that the right hand side correspond to the string diagram representing $\syn{\h}$ (see \cref{fig:markov_network_ex}). Now, the factors associated with the original network, defined in~\eqref{tab:misconception_factors}, yield a hypergraph functor $\Phi \colon \syn{\h}\to \mat$.
Postcomposing $t$ with $\Phi$ sets the following $\mat$-semantics for $\syn{\tr{\h}}$-generators: $f_D (D \given AC) \coloneqq \phi_D(D) \phi_{AD}(AD) \phi_{CD}(CD)$, $f_C (C \given AB) \coloneqq \phi_{BC}(BC) \phi_C(C)$, $f_B (B \given A)\coloneqq \phi_B(B) \phi_{AB}(AB)$ and $f_A(A)\coloneqq \phi_A(A)$. (For simplicity, here we use the same notation for the generators in the syntax categories and their image in $\mat$). 
Observe that, by definition, $f_C(C\given AB)$ does not really depend on $A$, but this additional input is forced by the triangulation. 
We shall see how this plays an important role in our next goal, which is to derive a functor $\cdsyn{\tr{\h}} \to \finstoch$ from $\syn{\tr{h}} \xrightarrow{t} \syn{\h} \xrightarrow{\Phi}\mat$ using a normalisation procedure (given in full generality in \cref{sec:functoriality_proofs}, \cref{lem:triangulation_finstoch}). 

Starting from the vertex $D$, we define the normalisation coefficient $\lambda_D$ which depends on $A$ and $C$, given by $\lambda_D(AC)=f_D(d\given AC) + f_D(\no{d}\given AC)$.
Therefore, $f_D (D \given AC)=\lambda_D(AC) g_D (D\given AC)$ for some stochastic matrix $g_D$ in $\finstoch$.
The idea is that $g_D$ will be the image of the generator $A\otimes C \to D$ in $\cdsyn{\tr{\h}}$.

As our aim is to obtain the same distribution associated to the network, we cannot dismiss the normalisation coefficient $\lambda_D(AC)$, and we therefore consider it together with $f_C(C\given AB)$, since $f_C$ is the only morphism where $C$ appears among the remaining $\lbrace f_A,f_B,f_C \rbrace$.
We now define a normalisation coefficient $\lambda_C (AB)= \lambda_D(A,c) \phi_{BC}(B,c) + \lambda_D(A,\no{c}) \phi_{BC}(B,\no{c})$ and obtain $\lambda_D(AC)f_C(C\given AB) = \lambda_C(AB) g_C(C\given AB)$, where $g_C$ is a stochastic matrix.
We note that $g_C$ does depend on $A$ although $f_C$ did not: in particular, the additional input of $f_C$ allows for the construction of $g_C$ without changing the type $A\otimes B \to C$.

Now, $\lambda_C(AB)$ depends on $B$, so we consider it with $f_B (B|A)$.
As above, we can define the normalisation coefficient $\lambda_B(A)$ and obtain $\lambda_C(AB)f_B(B\given A) = \lambda_B(A) g_B(B\given A)$ for some stochastic matrix $g_B$.
Similarly for $A$, we set a normalisation coefficient $\lambda_A$, which does not depend on any variable, such that $\lambda_B(A)f_A(A) = \lambda_A g_A$ for some stochastic matrix $g_A$.
%
By construction, the distribution $\prod_{v\in V_{\tr{\h}}} f_v (v\given \parents{v})$ associated to $ \syn{\tr{\h}}\to \mat$ is equal to $\lambda_A \prod_{v\in V_{\tr{\h}}} g_v(v\given \parents{v}) \propto \prod_{v\in\tr{\h}} g_v(v\given \parents{v})$, so the CD-functor $\cdsyn{\tr{\h}}\to \finstoch$ obtained by the family $\lbrace g_v\rbrace $ does preserve the distribution associated to the network.
The values resulting from this procedure are written below.
\begin{equation*}
{\begin{tabular}{|c|c|c|}
\hline
$A$ & $C$ & $g_D(d\given AC)$\\
\hline \hline
$a$ & $c$ & 0.5\\
\hline
$a$&$\no{c}$& 0.9999\\
\hline
$\no{a}$&$c$& 0.0001\\
\hline
$\no{a}$&$\no{c}$&0.5\\
\hline
\end{tabular}}
\hspace{0.3cm}
{\begin{tabular}{|c|c|c|}
\hline
$A$ & $B$ & $g_C(c\given AB)$\\
\hline \hline
$a$ & $b$ & 0.6666\\
\hline
$a$&$\no{b}$& 0.0002\\
\hline
$\no{a}$&$b$& 0.9998\\
\hline
$\no{a}$&$\no{b}$&0.3334\\
\hline
\end{tabular}}
\hspace{0.3cm}
{\begin{tabular}{|c|c|}
\hline
$A$ & $g_B(b\given A)$\\
\hline \hline
$a$ & 0.2307 \\
\hline
$\no{a}$& 0.8475\\
\hline
\end{tabular}}
\hspace{0.3cm}
{\begin{tabular}{|c|}
\hline
$g_A(a)$\\
\hline \hline
0.1806 \\
\hline
\end{tabular}}
\end{equation*}
\end{example}

We conclude by making some observation on the interaction of $\tr{-}$ and $\mor{-}$. 

\begin{proposition}\label{prop:trmor_mortr}
	There are two natural transformations: $\tr{\mor{-}} \to \id_{\bn}$ and $\mor{\tr{-}} \to \id_{\mn}$, where $\id_{-}$ indicates the identity functor.
\end{proposition}

The actions of these transformations are very simple: in $(\alpha,\eta) \colon \tr{\mor{\omega,\g}} \to (\omega,\g)$, the contravariant $\alpha \colon \g \to \tr{\mor{\g}}$ is a graph embedding, and $\eta \colon \omega \to \omega$ is the identity; the transformation $\mor{\tr{-}} \to \id_{\mn}$ is similarly defined. A more interesting observation is that $\tr{\mor{\g}}$ may be regarded as a \emph{chordal} representation of $\g$ (see \cite[Definitions 2.24 and 2.25]{koller2009probabilistic}). Finally, the reader may wonder whether an adjunction may be achieved, but this is too much to ask as it would require to have either natural $\id_{\bn}\to \tr{\mor{-}}$ or $\id_{\mn}\to \mor{\tr{-}}$.
	However, this is not possible already at the level of graphs, because in general $\tr{\mor{\g}}$ (resp.\ $\mor{\tr{\h}}$) has more edges than $\g$ (resp.\ $\h$), and represents less independencies. {We leave the details to \cref{sec:functoriality_proofs}, \cref{prop:noadjunction}.}
\section{Conclusions}\label{sec:conclusions}

We provided a categorical framework for Bayesian and Markov networks, focussing on translations between them in the form of moralisation and triangulation. In the spirit of Lawvere's categorical algebra, we characterised networks as functors from a `syntax' (the graph) to a `semantics' (its probabilistic interpretation). The translations were formulated abstractly by functor pre-composition, and defined inductively on the diagrammatic syntax.

A first direction for future work is accounting for additional PGMs. Our separation of syntax and semantics offers a clear way forward characterising several existing models, including Gaussian networks (where a discrete probabilistic semantics is replaced with a Gaussian one), networks based on `partial DAGs' (where the syntax allows both for directed and undirected edges), hidden Markov models (where not all vertices of the graph are `visible' to the string diagram interfaces), and factor graphs (where the syntax describes bipartite graphs). It also abstracts and simplifies reasoning on translations between models, whose specification in the existing literature is often in natural language and prone to ambiguity.

A second direction is to use our abstract account of moralisation and triangulation to study the mathematical structure of the algorithms where these transformations are employed, such as junction tree algorithms, clique tree message passing,  variable elimination, and graph-based optimisation. Ultimately, the aim is to offer a compositional perspective on these algorithms, leveraging our syntactic description of moralisation and triangulation.

\bibliography{references}

\newpage



\appendix 
\section{On Bayesian networks}\label[app]{sec:graphhom_bn}
Our main aim for this section is to prove that $\cdsyn{}$ is a functor (\cref{thm:graphhom-cd}). 
We start with the explicit definition of the category of ordered DAGs.
\begin{definition}\label{def:odag}
	The category $\odag$ of ordered DAGs is defined as follows:
	\begin{itemize}
		\item Objects are ordered DAGs, i.e.\ DAGs equipped with a total order such that $v \to w$ implies $v<w$;
		\item Morphisms are order-preserving graph homomorphisms, i.e.\ $\alpha \colon \g_1 \to \g_2$ is a function on the sets of vertices $V_{\g_1}\to V_{\g_2}$ that preserves the order and the edges.\footnote{If two vertices $x$ and $y$ have the same image under $\alpha$, we assume that the possible edges $(x,y)$ and $(y,x)$ are respected.}
	\end{itemize}
\end{definition}

\begin{notation}
	For any morphism $\phi$ in a CD-category, let us denote by $\In{\phi}$ and $\Out{\phi}$ the sets of inputs and outputs.
	Similarly, for any set of morphisms $S$, we consider $\Out{S}\coloneqq \bigcup_{\phi \in S} \Out{\phi}$, whereas $\In{S} \coloneqq \bigcup_{\phi \in S} \In{\phi}\cap \Out{S}^c$.
\end{notation}

The reason of the asymmetric notation for $\In{S}$ is motivated by the following notion.
\begin{definition}
	Let $\g$ be an ordered DAG. 	
	Consider a set of generators $S\subseteq\Sigma_{\g}$ and let $\phi$ be the generator in $S$ whose output is the biggest element of $\Out{S}$.
	By induction, we define $\copycomp{\emptyset}\coloneqq \id_I$ and 
	\[
	\scalebox{0.8}{\input{copycomp_S.tikz}}\qquad \coloneqq \qquad \scalebox{0.8}{\input{copycomp_def.tikz}}
	\] 
	where $S'\coloneqq S \setminus \lbrace \phi \rbrace$, and $\pi\colon \In{S'} \cup \Out{S'}\to (\In{S'} \cup \Out{S'})\cap \In{\phi}$
	is the marginalization (i.e. it deletes all the occurring vertices that are not inputs of $\phi$).
	In the string diagram we omitted the permutations of the inputs to avoid using additional notation.
	For a given $S$, we refer to $\copycomp{S}$ as the \newterm{copy-composition} of $S$.
\end{definition}
By definition, if $S=\lbrace \phi \rbrace$, then $\copycomp{S}=\phi$. 

\begin{remark}\label{rem:smithe}
	Although similar, the copy-composition above differs from the one introduced by Smithe in \cite{smithe2024copycomposition}.
	Indeed, for our purposes, copy-composition must allow inputs and outputs to simply overlap without entirely matching, which extends beyond Smithe's original notion. 
	On the other hand, this relaxed notion requires additional care as everything depends on a specific order. 
	This is why we have stated it only for sets of generators $S \subseteq \Sigma_{\g}$,  making it specific to the syntax categories $\cdsyn{\g}$ (and $\syn{\g}\coloneqq \freehyp{V_{\g},\Sigma_{\g}}$).
\end{remark}

\begin{remark}[Semantics with the Copy-Composition]\label{rem:semantics_copycomp}
	Consider the DAG $\begin{tikzpicture}
	\begin{pgfonlayer}{nodelayer}
		\node [style=bw] (15) at (-1.5, 0) {$A$};
		\node [style=bw] (31) at (1.5, 0) {$B$};
	\end{pgfonlayer}
	\begin{pgfonlayer}{edgelayer}
		\draw [style=arrow] (15) to (31);
	\end{pgfonlayer}
\end{tikzpicture}
$, and a CD-functor $F\colon \cdsyn{\g}\to \finstoch$. Set $\omega\coloneqq F(I\to A)$ and $f\coloneqq F(A \to B)$. 
	Then $F(\copycomp{\Sigma_{\g}}) = \input{fcopyw.tikz}$, which is the probability distribution given by $A \tensor B \owns (a,b)\mapsto f(b\given a) \omega(a)$, where the occurrences of $a$ as output of $\omega$ and as input of $f$ are identified. 
	In other words, $\input{fcopyw.tikz} = f(B\given A)\omega(A)$.

	In general, the interpretation of copy-composition, particularly when given a CD-functor $\cdsyn{\g}\to \finstoch$, is to identify all occurrences of the same variable. 
\end{remark}

\begin{lemma}\label{lem:copycomp}
	Let $S$ and $T$ be two sets of generators such that for each $\phi\in S$ and $\psi\in T$, $\Out{\phi} <\Out{\psi}$. Then 
	\[
		\scalebox{0.8}{\input{copycomp_ST.tikz}} \qquad = \qquad \scalebox{0.8}{\input{copycomp_SandT.tikz}}
	\]
	where $\pi$ marginalizes all the inputs and outputs of $S$ that are not inputs of $T$.
\end{lemma}
\begin{proof}
	By induction, we can assume that statement holds for $T'\coloneqq T \setminus \lbrace \phi\rbrace$, where $\phi$ is the generator whose output is the biggest element. Then 
	\[
	\begin{aligned}
		\scalebox{0.8}{\input{copycomp_ST.tikz}}\quad &=\quad \scalebox{0.8}{\input{copycomp_STprime_phi.tikz}}\\
		&=\quad \scalebox{0.8}{\input{copycomp_S_Tprime_phi.tikz}}\\
		&=\quad \scalebox{0.8}{\input{copycomp_S_T.tikz}}
	\end{aligned}
	\]
	where $\pi=\pi_1 \otimes \pi_2$, $\pi_1$ only considers the inputs obtained from $S$ and $\pi_2$ those obtained from $T'$. Since the dashed box corresponds to $\copycomp{T}$, the statement is shown.
\end{proof}

\begin{notation}
	Let $\alpha\colon \g \to \g'$ be an order-preserving graph homomorphism, and let $v \in V_{\g'}$.
	We write $S_v^{\alpha}$, or $S_v$ when there is no confusion, for the set of generators $\phi \in \Sigma_{\g}$ such that $\Out{\phi} \subseteq \alpha^{-1}(v)$.
\end{notation}
\begin{remark}\label{rem:sv_properties}
	Let us highlight some important properties of $S_v$.
	\begin{enumerate}
		\item\label{it:sv_out} As every vertex in $V_{\g}$ is an output for some generator in $\Sigma_{\g}$, $\Out{S_v}=\alpha^{-1}(v)$.
		\item\label{it:sv_partition} The collection $\lbrace S_v \mid v\in \alpha(\g) \rbrace$ gives a partition of $\Sigma_{\g}$ because the same holds for the preimage.
		\item\label{it:sv_in_all} We have 
		\[
			\bigcup_{\phi \in S_v} \In{\phi} = \bigcup_{\phi \in S_v} \parents{\Out{\phi}} = \bigcup_{w \in \alpha^{-1}(v)} \parents{w} = \parents{\alpha^{-1}(v)},
		\]
		and therefore $\In{S_v}=\parents{\alpha^{-1}(v)}\cap \alpha^{-1}(v)^c$ by \cref{it:sv_out}. In particular, from this equality we infer that $\In{S_v}\subseteq \alpha^{-1}(\parents{v})$.
	\end{enumerate}
\end{remark}

We are now ready to define the association on morphisms. 
We will tacitly use \cref{rem:sv_properties} in the definition.

\begin{definition}\label{def:cdsyn_morph}
	Let $\alpha \colon \g \to \g'$ be an order-preserving graph homomorphism.
	We then define a functor $\cdsyn{\alpha}\colon \cdsyn{\g'}\to \cdsyn{\g}$ as follows.
	\begin{itemize}
		\item On objects, $v$ is sent to the tensor product given by $\alpha^{-1}(v)$ (and ordered according to the order on $\g$).
		\item On morphisms, $\cdsyn{\alpha}$ is given by 
		\[
			\scalebox{0.8}{\input{parents_morph.tikz}} \quad \longmapsto \quad \scalebox{0.8}{\input{synf_dags.tikz}}
		\]
		where $\In{S_v}^c\coloneqq \alpha^{-1}(\parents{v})\setminus \In{S_v}$. 
		As in the previous definition, we omit some permutation of objects to avoid additional notation.
	\end{itemize}
\end{definition}

\begin{lemma}\label{lem:copycomp2}
	Let $\alpha \colon \g \to \g'$ be an order-preserving graph homomorphism and let $T\subset\Sigma_{\g'}$. Then 
	\[
	\cdsyn{\alpha}\left(\scalebox{0.8}{\input{copycompT.tikz}}\right) \qquad =\qquad  \scalebox{0.8}{\input{synf_dags_lem.tikz}}
	\]
\end{lemma}
\begin{proof}
	By induction, let us assume the statement is true for $T'= T \setminus\lbrace \phi\rbrace$, where $\phi$ is the generator whose output is the biggest element. Then
	\[
	\begin{aligned}
		\cdsyn{\alpha}\left(\scalebox{0.8}{\input{copycompT.tikz}}\right) \quad &=\quad \cdsyn{\alpha}\left(\scalebox{0.8}{\input{copycomp_Tprime_phi.tikz}}\right)\\
		&=\quad \scalebox{0.8}{\input{copycomp_Tprimephi_im.tikz}}
	\end{aligned}
	\] 
	and we conclude by \cref{lem:copycomp}.
\end{proof}

\begin{proof}[Proof of \cref{thm:graphhom-cd}]
	We want to prove that the association 
	\[
	\begin{array}{rccc}
		\cdsyn{} \colon& \odag & \to & \cdcat\\
		& \g,\quad \alpha \colon \g \to \g' &\mapsto & \cdsyn{\g},\quad \cdsyn{\alpha}\colon \cdsyn{\g'}\to \cdsyn{\g}
	\end{array}
	\]
	where $\cdsyn{\alpha}$ is defined as in \cref{def:cdsyn_morph}, is a contravariant functor.
	We note that indeed $\cdsyn{\id}=\id$, so we are left to prove composition.
	Let $\alpha\colon \g_1 \to \g_2$ and $\beta\colon \g_2\to \g_3$. 
	On objects, $\cdsyn{\alpha}\cdsyn{\beta}(v) = \cdsyn{\beta\alpha} (v)$ holds by properties of the preimage and the fact that $\beta$ and $\alpha$ are order-preserving.
	
	Let us now consider a generator $\input{parents_morph.tikz}$ in $\cdsyn{\g_3}$. 
	By \cref{lem:copycomp2}, it suffices to show that $\bigcup_{\phi \in S_v^{\beta}} {S_{\Out{\phi}}^{\alpha}} = {S_v^{\beta\alpha}}$. But this follows immediately from \cref{rem:sv_properties}:
	\[
		\bigcup_{\phi \in S_v^{\beta}} S_{\Out{\phi}}^{\alpha} = \bigcup_{w \in \beta^{-1}(v)} S_w^{\alpha} = \lbrace \phi \mid \Out{\phi} \subseteq \alpha^{-1}(w) \text{ for some }w \in \beta^{-1}(v) \rbrace = {S_v^{\beta\alpha}}.
	\]
\end{proof}

\begin{remark}[Factorising Distributions]\label{rem:factorising_distributions_DAG}
	An important morphism is the contraction $\g \to \bullet$, which we often use in the main text. 
	Via the functor $\cdsyn{}$, this contraction gives the CD-functor $!_{\g}\colon \cdsyn{\bullet}\to \cdsyn{\g}$, which sends the single generator $I\to \bullet$ to $\copycomp{\Sigma_{\g}}$.

	When considering a CD-functor $F\colon \cdsyn{\g}\to \finstoch$, the distribution $\omega$ associated to $\cdsyn{\bullet}\to \cdsyn{\g}\to \finstoch$ is then defined by $F(\copycomp{\Sigma_{\g}})$.
	Since CD-functors preserve $\input{copy.tikz}$ and $\begin{tikzpicture}
	\begin{pgfonlayer}{nodelayer}
		\node [style=bn] (35) at (0.75, 0) {};
		\node [style=none] (36) at (-0.75, 0) {};
	\end{pgfonlayer}
	\begin{pgfonlayer}{edgelayer}
		\draw (35) to (36.center);
	\end{pgfonlayer}
\end{tikzpicture}
$, we have a factorisation of $\omega$. 
	More explicitly, set $f_v\coloneqq F(\input{parents_morph.tikz})$.
	Then, $\omega(V_{\g})= \prod_{v} f_v(v\given \parents{v})$, where we identified all occurrences of the same variable, as discussed in \cref{rem:semantics_copycomp}.
\end{remark}

\begin{proof}[Proof of \cref{prop:irredundantbayesian_functor}]
	We note that whenever we work with free categories, functors can be freely defined by describing where the generating objects and morphisms are sent. 
	
	Let $P$ be an irredundant Bayesian network over $\g$, and consider the CD-functor $\cdsyn{\bullet}\to \finstoch$ given by sending $\bullet$ to $X$ and $\begin{tikzpicture}
	\begin{pgfonlayer}{nodelayer}
		\node [style=small box] (0) at (0, 0) {};
		\node [style=none] (1) at (2, 0) {};
	\end{pgfonlayer}
	\begin{pgfonlayer}{edgelayer}
		\draw (1.center) to (0);
	\end{pgfonlayer}
\end{tikzpicture}
$ to the probability distribution $P$.
	By definition, $P$ factorises as a product $\prod_v f_v$ for some stochastic matrices $f_v$ and some assignment $\tau$. 
	We define $F\colon \cdsyn{\g}\to \finstoch$ by sending $v\mapsto \tau(v)$ and the generator $\input{parents_morph.tikz}$ to $f_v$. 
	\cref{rem:factorising_distributions_DAG} is sufficient to ensure that the CD-functor $\cdsyn{\bullet} \to \finstoch$ given by $P$ indeed factorises through $!_{\g}$.
	
	Conversely, if $\cdsyn{\bullet}\to \finstoch$ admits a factorisation through $!_{\g}$, the image $P$ of $I\to \bullet$ can be rewritten using the images $f_v$ of the generators $\input{parents_morph.tikz}$ again by \cref{rem:factorising_distributions_DAG}, and so we obtain an irredundant Bayesian network. 
	As the two constructions are clearly the inverse of one another, the wanted bijection holds.
\end{proof}

\section{On Markov networks}\label[app]{sec:graphhom_mn}

\begin{proof}[Proof of \cref{prop:markov_functor}]
	Whenever we have a Markov network, we can define the hypergraph functor on generators by setting $
	F(v)\coloneqq \tau(v)$ and $F\left(\begin{tikzpicture}
	\begin{pgfonlayer}{nodelayer}
		\node [style=small box] (0) at (0, 0) {};
		\node [style=none] (1) at (1.5, 0) {};
		\node [style=none] (2) at (2, 0) {$C$};
	\end{pgfonlayer}
	\begin{pgfonlayer}{edgelayer}
		\draw (1.center) to (0);
	\end{pgfonlayer}
\end{tikzpicture}
\right) \coloneqq \phi_C$.
	Conversely, given a hypergraph functor $F \colon \syn{\mathcal{H}} \to \mat$, one simply uses the definitions above in the converse direction: $
	\tau(v) \coloneqq F(v)$ and $\phi_C \coloneqq F\left(\begin{tikzpicture}
	\begin{pgfonlayer}{nodelayer}
		\node [style=small box] (0) at (0, 0) {};
		\node [style=none] (1) at (1.5, 0) {};
		\node [style=none] (2) at (2, 0) {$C$};
	\end{pgfonlayer}
	\begin{pgfonlayer}{edgelayer}
		\draw (1.center) to (0);
	\end{pgfonlayer}
\end{tikzpicture}
\right)$.
	As the two functions defined are clearly the inverse of one another, the statement follows.
\end{proof}

We now discuss some preliminaries to prove \cref{thm:graphhom-hyp}.

\begin{definition}
	The category $\ougr$ of ordered undirected graphs is defined as follows:
	\begin{itemize}
		\item Objects are ordered undirected graphs;
		\item Morphisms are order-preserving graph homomorphisms, i.e.\ $\alpha \colon \h_1 \to \h_2$ is a function on the sets of vertices $V_{\h_1}\to V_{\h_2}$ that preserves the order and the edges.\footnote{If two vertices $x$ and $y$ have the same image under $f$, we assume that the possible edge $\lbrace x,y\rbrace$ is respected.}
	\end{itemize}
\end{definition}

Similarly to \cref{sec:graphhom_bn}, we need to define what it means to consider several morphisms together. 
Importantly, as here we do not need to take care of directionality, this definition can be extended to a general situation, which will be important in \cref{sec:functoriality_proofs}.
\begin{definition}
	A morphism $\phi \colon I\to C$ in a hypergraph category is a \newterm{factor} if it comes equipped with a sequence of objects $(X_1 ,\dots , X_n)$ such that $X_1 \tensor \cdots \tensor X_n=C$. 
	
	We write $\Out{\phi}\coloneqq \lbrace X_1, \dots, X_n\rbrace$, and, more generally, $\Out{S}\coloneqq \bigcup_{\phi \in S} \Out{\phi}$ for a finite set of factors $S$. 
\end{definition}

A finite set of factors $S$ is assumed to come with an ordering of $\Out{S}$ to avoid the swapping issue discussed in \cref{rem:ordering}.

In the categories $\syn{\g}$ and $\syn{\h}$, every morphism $I\to C$ is tacitly assumed to become a factor by considering the vertices of the associated graph. In particular, every set of factors in these categories comes with an induced ordering.

\begin{definition}
	For every finite set of factors $S$ in a hypergraph category, we define by induction $\compcomp{\emptyset}=\id_I$ and
	\[
		\scalebox{0.8}{\begin{tikzpicture}
	\begin{pgfonlayer}{nodelayer}
		\node [style=morphism] (79) at (-5.75, 0) {$\compcomp{S}$};
		\node [style=none] (80) at (-3, 0) {};
		\node [style=none] (81) at (-5.25, 0) {};
		\node [style=none] (83) at (-2, 0) {$\Out{S}$};
	\end{pgfonlayer}
	\begin{pgfonlayer}{edgelayer}
		\draw (81.center) to (80.center);
	\end{pgfonlayer}
\end{tikzpicture}
} \qquad \coloneqq \qquad \scalebox{0.8}{\input{compare_comp_def.tikz}}
	\]
	where $\phi$ is arbitrarily chosen and $S'\coloneqq S \setminus \lbrace \phi \rbrace$. Permutations of the outputs on the right hand side are omitted for brevity. 
	For a given $S$, we refer to $\compcomp{S}$ as the \newterm{compare-composition} of $S$.
\end{definition}
\begin{lemma}\label{lem:compcomp}
	The compare-composition is well-defined, i.e. it does not depend on the choice of $\phi$. 
	Moreover, for any disjoint sets of factors $S$ and $T$,
	\[
		\scalebox{0.8}{\begin{tikzpicture}
	\begin{pgfonlayer}{nodelayer}
		\node [style=morphism] (79) at (0, 0) {$\compcomp{S\cup T}$};
		\node [style=none] (80) at (3, 0) {};
		\node [style=none] (81) at (0.5, 0) {};
		\node [style=none] (82) at (3.5, 0) {};
	\end{pgfonlayer}
	\begin{pgfonlayer}{edgelayer}
		\draw (81.center) to (80.center);
		\draw (80.center) to (82.center);
	\end{pgfonlayer}
\end{tikzpicture}
}\qquad =\qquad \scalebox{0.8}{\input{compare_comp2.tikz}}	
	\]
\end{lemma}
The proof is omitted as it immediately follows from associativity and commutativity of the compare morphism.

\begin{remark}[Semantics with the Compare-Composition]\label{rem:semantics_compcomp}
	\cref{rem:semantics_copycomp} can be translated to this setting, meaning that the compare-composition identifies all occurrences of the same variable on the different factors. 
	For example, take two factors $\phi\colon I \to A \tensor B$ and $\psi\colon I \to A \tensor C$ in $\mat$. 
	Then $\compcomp{\lbrace \phi,\psi\rbrace}$ is given by the factor $A\tensor B\tensor C\owns (a,b,c)\mapsto \phi(a,b)\psi(a,c)$, also written as $\phi(AB)\psi(AC)$.
\end{remark}

\begin{definition}\label{def:syn_morph}
	Let $\alpha \colon \h \to \h'$ be an order-preserving graph homomorphism. The hypergraph functor $\syn{\alpha}\colon \syn{\h'} \to \syn{\h}$ is defined as follows:
	\begin{itemize}
	\item It maps an object $v$ to the tensor product given by $\alpha^{-1}(v)$.
	\item On morphisms, 
	\[
		\scalebox{0.8}{\begin{tikzpicture}
	\begin{pgfonlayer}{nodelayer}
		\node [style=small box] (0) at (0, 0) {};
		\node [style=none] (1) at (1.5, 0) {};
		\node [style=none] (2) at (2, 0) {$C$};
	\end{pgfonlayer}
	\begin{pgfonlayer}{edgelayer}
		\draw (1.center) to (0);
	\end{pgfonlayer}
\end{tikzpicture}
}\qquad \longmapsto \qquad \scalebox{0.8}{\input{synf_ugr.tikz}}
	\]
	where $S_C\coloneqq \lbrace\phi \colon I \to D \mid \alpha(D)=C\rbrace$.
	\end{itemize}
\end{definition}
\begin{lemma}\label{lem:compcomp2}
	Let $\alpha \colon \h \to \h'$ be an order-preserving graph homomorphism and let $T\subset \Sigma_{\h'}$. Then 
	\[
	\syn{\alpha}\left(\scalebox{0.8}{\begin{tikzpicture}
	\begin{pgfonlayer}{nodelayer}
		\node [style=morphism] (58) at (0, 0) {$\compcomp{T}$};
		\node [style=none] (61) at (0, 0) {};
		\node [style=none] (62) at (2.75, 0) {};
	\end{pgfonlayer}
	\begin{pgfonlayer}{edgelayer}
		\draw (61.center) to (62.center);
	\end{pgfonlayer}
\end{tikzpicture}
}\right) \qquad =\qquad  \scalebox{0.8}{\input{synf_ugr_lem.tikz}}
	\]
\end{lemma}
This is proven by induction using \cref{lem:compcomp}, following the approach of the proof of \cref{lem:copycomp2}.

\begin{proof}[Proof of \cref{thm:graphhom-hyp}]
	We aim to prove that the association 
	\[
	\begin{array}{rccc}
		\syn{} \colon& \ougr & \to & \hypcat\\
		& \h,\quad \alpha \colon \h \to \h' &\mapsto & \syn{\h},\quad \syn{\alpha}\colon \syn{\h'}\to \syn{\h}
	\end{array}
	\]
	where $\syn{\alpha}$ is defined in \cref{def:syn_morph}, is a contravariant functor.
	A direct check shows that $\syn{\id}=\id$, so we focus on composition. Let $\alpha\colon \h_1 \to \h_2$ and $\beta\colon \h_2 \to \h_3$.
	As in the proof of \cref{thm:graphhom-cd}, $\syn{\alpha}\comp \syn{\beta}(v)=\syn{\beta\alpha}(v)$ because the considered graph homomorphisms are order-preserving.

	Regarding composition, let $\phi \colon I \to C$ in $\syn{\h_3}$. By \cref{lem:compcomp2}, it suffices to show that $\bigcup_{\phi \in S_C^{\beta}} S_{\Out{\phi}}^{\alpha} = S_C^{\beta\alpha}$, where $S_Y^{\ell}\coloneqq \lbrace \psi \colon I \to X \mid \ell(X)=Y \rbrace$. 
	By unpacking the definition,
	\[
	\bigcup_{\phi \in S_C^{\beta}} S_{\Out{\phi}}^{\alpha} = \lbrace \psi \colon I \to D \,\given\,  \exists\,  E \in \clique{\h_2}\text{ such that } \alpha(D)=E\text{ and } \beta(E)=C \rbrace.
	\] 
	Since graph homomorphisms send cliques to cliques, whenever $D$ is a clique, $\alpha(D)$ is a clique as well. 
	Therefore, the condition on the right can be shortened to $\beta\alpha(D)=C$, so the wanted equality $\bigcup_{\phi \in S_C^{\beta}} S_{\Out{\phi}}^{\alpha} = S_C^{\beta\alpha}$ indeed holds.
\end{proof}

\begin{remark}[Factorising Distributions]\label{rem:factorising_distributions_UGr}
	Similarly to \cref{rem:factorising_distributions_DAG}, we consider the contraction map $\h \to \bullet$, which via $\syn{}$ yields the hypergraph functor $!_{\h}\colon \syn{\bullet}\to \syn{\h}$ sending the single generator $I\to \bullet$ to $\compcomp{\Sigma_{\h}}$.

	We now consider a hypergraph functor $\Phi\colon \syn{\h}\to \mat$ and set $\phi_C\coloneqq \Phi(\begin{tikzpicture}
	\begin{pgfonlayer}{nodelayer}
		\node [style=small box] (0) at (0, 0) {};
		\node [style=none] (1) at (1.5, 0) {};
		\node [style=none] (2) at (2, 0) {$C$};
	\end{pgfonlayer}
	\begin{pgfonlayer}{edgelayer}
		\draw (1.center) to (0);
	\end{pgfonlayer}
\end{tikzpicture}
)$.
	The distribution $\omega$ associated to $\syn{\bullet}\to \syn{\h}\to \mat$ is then defined by $\Phi(\compcomp{\Sigma_{\h}})$, and since hypergraph functors preserve $\input{compare.tikz}$, we have a factorisation $\omega(V_{\h})= \prod_{C} \phi_C(C)$, where we identified all occurrences of the same variable, as discussed in \cref{rem:semantics_compcomp}.
\end{remark}

\begin{proof}[Proof of \cref{prop:irredundantmarkov_functor}]
	Commutativity of \eqref{eq:mn_factorisation} is exactly requiring that 
	\begin{equation}\label{eq:Q}
		\begin{tikzcd}
			\cdsyn{\bullet} \ar[r,"\star"]\ar[rrr,bend left=20, "Q"] & \syn{\bullet}\ar[r,"!"] &  \syn{\h} \ar[r,"\Phi"] &\mat\ar[r,"q"] &\finprojstoch
		\end{tikzcd}
	\end{equation}
	is equal to $\cdsyn{\bullet}\xrightarrow{\omega} \finstoch \xrightarrow{i} \finprojstoch$. 
	As already discussed in the main text, the latter compositions with $q$ and $i$ respectively are simply to describe that $Q\propto \omega$, which holds if and only if there exists a normalisation coefficient $Z$ such that $\omega = \frac{1}{Z}Q$ (to prevent notation overload, here we use $Q$ and $\omega$ both as the functors and their associated distributions $Q(\begin{tikzpicture}
	\begin{pgfonlayer}{nodelayer}
		\node [style=small box] (0) at (0, 0) {};
		\node [style=none] (1) at (2, 0) {};
	\end{pgfonlayer}
	\begin{pgfonlayer}{edgelayer}
		\draw (1.center) to (0);
	\end{pgfonlayer}
\end{tikzpicture}
)$ and $\omega(\begin{tikzpicture}
	\begin{pgfonlayer}{nodelayer}
		\node [style=small box] (0) at (0, 0) {};
		\node [style=none] (1) at (2, 0) {};
	\end{pgfonlayer}
	\begin{pgfonlayer}{edgelayer}
		\draw (1.center) to (0);
	\end{pgfonlayer}
\end{tikzpicture}
)$).
	By \cref{rem:factorising_distributions_UGr}, we have that $Q(V_{\h}) = \prod_C \phi_C (C)$ with $\phi_C\coloneqq \Phi(\begin{tikzpicture}
	\begin{pgfonlayer}{nodelayer}
		\node [style=small box] (0) at (0, 0) {};
		\node [style=none] (1) at (1.5, 0) {};
		\node [style=none] (2) at (2, 0) {$C$};
	\end{pgfonlayer}
	\begin{pgfonlayer}{edgelayer}
		\draw (1.center) to (0);
	\end{pgfonlayer}
\end{tikzpicture}
)$ because of the factorisation, so we conclude that $\omega(V_{\h}) = \frac{1}{Z} \prod_C \phi_C (C)$, and thus $\omega$ is an irredundant Markov network.

	Conversely, if $\omega$ is an irredundant Markov network, then we can find an assignment $\tau$ and factors $\phi_C\colon I \to C$ such that $\omega(V_{\h})= \frac{1}{Z} \prod_C \phi_C(C)$, where $Z$ is a normalisation coefficient.
	Then, we define $\Phi\colon \syn{\h}\to \mat$ by sending $v$ to $\tau(v)$ and $\begin{tikzpicture}
	\begin{pgfonlayer}{nodelayer}
		\node [style=small box] (0) at (0, 0) {};
		\node [style=none] (1) at (1.5, 0) {};
		\node [style=none] (2) at (2, 0) {$C$};
	\end{pgfonlayer}
	\begin{pgfonlayer}{edgelayer}
		\draw (1.center) to (0);
	\end{pgfonlayer}
\end{tikzpicture}
$ to $\phi_C$.
	Moreover, $\omega$ gives rise to a CD-functor $\cdsyn{\bullet}\to \finstoch$ by $\begin{tikzpicture}
	\begin{pgfonlayer}{nodelayer}
		\node [style=small box] (0) at (0, 0) {};
		\node [style=none] (1) at (2, 0) {};
	\end{pgfonlayer}
	\begin{pgfonlayer}{edgelayer}
		\draw (1.center) to (0);
	\end{pgfonlayer}
\end{tikzpicture}
\mapsto \omega$.
	Then the commutativity of \eqref{eq:mn_factorisation} holds as it corresponds to $\omega(V_{\h})\propto \prod_C \phi_C(C)$, and this is true by definition of an irredundant Markov network.
\end{proof}

\section{Functoriality of the Moralisation and Triangulation Functors}\label[app]{sec:functoriality_proofs}
In this section we prove \cref{thm:moralisation,thm:triangulation}. 
To this end, it is important to show the connection between the copy and the compare composition. 

\begin{definition}
	Let $\cC$ be a hypergraph category. Given a morphism $f\colon X \to Y$, its \newterm{graph} is defined as $\input{graph_f_def.tikz}\coloneqq \input{fcapXY.tikz}$.
	For a finite set of morphisms $S$, we also write $\graph{S}$ for the set of graphs.
\end{definition}

\begin{lemma}\label{lem:graphcopycomp}
	Let $\g$ be an ordered DAG and let $S\subseteq \Sigma_{\g}$ be a set of generators.
	Then, in $\syn{\g}$, we have $\graph{\copycomp{S}}=\compcomp{\graph{S}}$.
\end{lemma}
In the statement we explicitly mention $\syn{\g}$ to point out that we cannot consider the case of CD-categories, where the compare-composition is not defined.
The proof is by induction, and it follows from the special Frobenius equations~\eqref{eq:hypergraph}.

\begin{remark}\label{rem:star_G}
For each ordered DAG $\g$, we have a CD-functor $\star_{\g} \colon \cdsyn{\g}\to \syn{\g}$ given by sending each generator to itself (recall that $\syn{\g}\coloneqq \freehyp{V_{\g},\Sigma_{\g}}$). 
In particular, the family $(\star_{\g})_{\g}$ is a natural transformation when we intepret $\cdsyn{}$ and $\syn{}$ as functors $\odag \to \cat{CDcat}$.

As already stated in the main text, any semantics CD-functor $F\colon \cdsyn{\g} \to \finstoch$ yields a hypergraph functor $\tilde{F}\colon \syn{\g}\to \mat$ defined on generators in the same way as $F$. 
We note that by definition, this modification satisfies the commutative diagram
\[
\begin{tikzcd}
	\cdsyn{\g} \ar[r,"\star_{\g}"]\ar[d,"F"] & \syn{\g}\ar[d,"\tilde{F}"]\\
	\finstoch \ar[r,hook] & \mat
\end{tikzcd}
\]
With the same strategy, we define $\tilde{!}_{\g}\colon \syn{\bullet}\to \syn{\g}$, the hypergraph functor such that $\cdsyn{\bullet}\xrightarrow{\star} \syn{\bullet}\xrightarrow{\tilde{!}_{\g}} \syn{\g}$ coincides with $\cdsyn{\bullet} \xrightarrow{!_{\g}}\cdsyn{\g} \xrightarrow{\star_{\g}} \syn{\g}$.
\end{remark}

\begin{proof}[Proof of \cref{thm:moralisation}]
	To ensure that the moralisation gives rise to a functor, we first prove that $\tilde{!}_{\g}\colon \syn{\bullet}\to \syn{\g}$ factors through $m\colon \syn{\mor{\g}}\to \syn{\g}$.
	Recall that $m$, defined in~\eqref{eq:moralisation_functor}, sends $\begin{tikzpicture}
	\begin{pgfonlayer}{nodelayer}
		\node [style=small box] (0) at (0, 0) {};
		\node [style=none] (1) at (1.5, 0) {};
		\node [style=none] (2) at (2, 0) {$C$};
	\end{pgfonlayer}
	\begin{pgfonlayer}{edgelayer}
		\draw (1.center) to (0);
	\end{pgfonlayer}
\end{tikzpicture}
$ to the graph $\input{graph_parents_morph.tikz}$ whenever $C=\lbrace v \rbrace \cup \parents{v}$, and to $\begin{tikzpicture}
	\begin{pgfonlayer}{nodelayer}
		\node [style=bn] (24) at (-0.75, 0) {};
		\node [style=none] (25) at (0.75, 0) {};
	\end{pgfonlayer}
	\begin{pgfonlayer}{edgelayer}
		\draw (24) to (25.center);
	\end{pgfonlayer}
\end{tikzpicture}
$ otherwise.
	
	We study the composition $\syn{\bullet}\xrightarrow{!_{\mor{\g}}} \syn{\mor{\g}}\xrightarrow{m} \syn{\g}$. 
	By definition, $!_{\mor{\g}}$ sends the unique generator $I \to \bullet$ to $\compcomp{\Sigma_{\mor{\g}}}$.
	Before discussing how $m$ acts on $\compcomp{\Sigma_{\mor{\g}}}$, since the moralisation connects parents, $\lbrace v \rbrace \cup \parents{v}$ is a clique in $\mor{\g}$ for every $v\in V_{\g}$, i.e.\ all $\input{graph_parents_morph.tikz}$ are in the image of $m$.
	We conclude that $m(\compcomp{\Sigma_{\mor{\g}}})=\compcomp{\graph{\Sigma_{\g}}}$, because the compare morphisms are preserved by hypergraph functors. 
	By \cref{lem:graphcopycomp}, $\compcomp{\graph{\Sigma_{\g}}}= \graph{\copycomp{\Sigma_{\g}}}$.
	Moreover, $\copycomp{\Sigma_{\g}}$ has trivial input, so it coincides with its graph:  $\compcomp{\graph{\Sigma_{\g}}}= \copycomp{\Sigma_{\g}}$.
	On the other side, $\tilde{!}_{\g}\colon \syn{\bullet}\to \syn{\g}$ sends $I\to \bullet$ to $\copycomp{\Sigma_{\g}}$, so the wanted factorisation holds.

	Now, whenever we have a Bayesian network $(\omega,\g)$, consider some $F\colon \cdsyn{\g}\to \finstoch$ factorising $\omega$. Then the associated Markov network $(\omega,\mor{\g})$ is obtained by the commutative diagram (already appeared in the main text)
	\begin{equation}
		\begin{tikzcd}
			\syn{\mor{\g}}\ar[r,"m"]& \syn{\g} \ar[r,"\tilde{F}"] &\mat\ar[dr,"q"]&  \\
			\syn{\bullet} \ar[u,"!_{\mor{\g}}"]\ar[ru,"\tilde{!}_{\g}"] &\cdsyn{\g}\ar[u,"\star_{\g}"]\ar[dr,"F"] && \finprojstoch\\
			& \cdsyn{\bullet} \ar[lu,"\star"]\ar[r,"\omega"]\ar[u,"!_{\g}"] & \finstoch\ar[uu,hook]\ar[ur,"i"]&
		\end{tikzcd}
	\end{equation}
	where we used \cref{rem:star_G} as well as the fact that $\finstoch \hookrightarrow \mat \xrightarrow{q} \finprojstoch$ coincides with $i$ by direct check.
	
	We are left to discuss how the moralisation acts on morphisms. 
	Let $\alpha \colon \g' \to \g$ be an order-preserving graph homomorphism. 
	Then $\alpha$ can be seen as an order-preserving graph homomorphism $\mor{\g'}\to \mor{\g}$, as stated in the main text. 
	This is easily checked: every edge $v\edge w$ in $\mor{\g'}$ either comes from an edge of $\g'$, so it is obviosly preserved by $\alpha$, or it is given by a certain vertex $u$ such that $v\to u$ and $w\to u$. Since $\alpha$ sends edges to edges, $\alpha(v)$ and $\alpha(w)$ are parents (or coincide with) $\alpha(u)$. 
	This means that we have an edge $\alpha(v)\edge \alpha(w)$ in $\mor{\g}$ (or $\alpha(v)=\alpha(w)$).
	We can then define the moralisation functor on morphisms by sending $(\alpha,\eta)\colon (\omega,\g)\to (\omega',\g')$ to $(\alpha,\eta)\colon (\omega,\mor{\g})\to (\omega',\mor{\g'})$.
\end{proof}

\begin{proof}[Proof of \cref{thm:triangulation}]
	Analogously to the previous proof, we start by showing that $\syn{\bullet}\to\syn{\h}$ factors through $t\colon \syn{\tr{\h}}\to \syn{\h}$ as defined in~\eqref{eq:triangulation}. 

	By definition, $\tilde{!}_{\tr{\h}}\colon \syn{\bullet}\to \syn{\tr{\h}}$ sends the unique generator $I\to\bullet$ to the copy-composition $\copycomp{\Sigma_{\tr{\h}}}$.
	As we noted in the previous proof, $\copycomp{\Sigma_{\tr{\h}}}=\graph{\copycomp{\Sigma_{\tr{\h}}}}$ because it has trivial input. Therefore, by \cref{lem:graphcopycomp}, $\copycomp{\Sigma_{\tr{\h}}} = \compcomp{\graph{\Sigma_{\tr{\h}}}}$. 
	The image of this morphism via $\syn{\tr{\h}}\to \syn{\h}$ is then given by $\compcomp{\bigcup_{v \in V_{\tr{\h}}}\compcomp{C_v}}$, where $C_v 
	= \lbrace \phi \colon I\to C \mid v\in C\subseteq \parents{v} \cup \lbrace v \rbrace\rbrace \subseteq \Sigma_{\h} $.
	Here, the additional tensoring seemingly needed with $\begin{tikzpicture}
	\begin{pgfonlayer}{nodelayer}
		\node [style=bn] (24) at (-0.75, 0) {};
		\node [style=none] (25) at (0.75, 0) {};
		\node [style=none] (26) at (1.25, 0) {$Q$};
	\end{pgfonlayer}
	\begin{pgfonlayer}{edgelayer}
		\draw (24) to (25.center);
	\end{pgfonlayer}
\end{tikzpicture}
$ (where $Q$ is defined as in~\eqref{eq:triangulation}) can be disregarded because for every vertex $v$, the fact that $\lbrace v \rbrace\in C_v$ ensures that $v$ is always considered in $\compcomp{\bigcup_{v \in V_{\tr{\h}}}\compcomp{C_v}}$.
	Moreover, $\compcomp{\bigcup_{v \in V_{\tr{\h}}}\compcomp{C_v}} = \compcomp{\bigcup_{v \in V_{\tr{\h}}} C_v}$ by \cref{lem:compcomp} because the sets $C_v$ are all disjoint: if a clique $D$ belongs to $C_v\cap C_w$, then $v,w \in D$ and both of them are parents of each other, which contradicts the fact that $\tr{\h}$ is a DAG.

	We claim that $\bigcup_{v \in V_{\tr{\h}}}{C_v}=\Sigma_{\h}$, where $\subseteq$ holds by definition. 
	Let us write $\bigcup_{v \in V_{\tr{\h}}}{C_v}$ more explicitly:
	\[
		\bigcup_{v \in V_{\tr{\h}}}{C_v}= \left\lbrace \phi\colon I \to C \given \exists\, v\in V_{\tr{\h}} \text{ such that }v\in C\subseteq \parents{v}\cup\lbrace v\rbrace \right\rbrace.
	\]
	Now, given any clique $C\in \clique{\h}$, let us consider the biggest element $v\in C$. 
	Then for any element $w \in C$, there must exist an edge $w \edge v$ because $C$ is a clique.
	By assumption, $w\le v$, and therefore $w\to v$ in $\tr{\h}$. 
	We conclude that $v \in C\subseteq \parents{v}\cup\lbrace v\rbrace$, so $\phi \colon I \to C$ belongs to $C_v$.  
	As claimed, $\bigcup_{v \in V_{\tr{\h}}}{C_v}=\Sigma_{\h}$, and therefore 
	$\syn{\bullet}\to \syn{\tr{\h}}\to \syn{\h}$ sends $I\to \bullet$ to $\compcomp{\Sigma_{\h}}$.
	Since $!_{\h}\colon \syn{\bullet}\to \syn{\h}$ is defined by the same assignment, we have shown that $!_{\h}$ indeed factorises through $t$.
	Finally, any $(\omega,\h)\in \mn$ leads to $\omega$ being factored through $\syn{\tr{\h}}$ in the following way:
	\[
	\begin{tikzcd}[row sep=1pt,column sep=large]
		\cdsyn{\bullet}\ar[r,"!_{\tr{\h}}"]\ar[dd,"\omega"]&\cdsyn{\tr{\h}} \ar[r,"\star_{\tr{\h}}"]& \syn{\tr{\h}}\ar[dr,"t"]\ar[dd]&\\
		&&& \syn{\h}\ar[dl]\\
		\finstoch\ar[r,"i"] & \finprojstoch & \mat\ar[l,"q" above]&
	\end{tikzcd}
	\]
	However, we want to show that this is sufficient to ensure the existence of $\cdsyn{\tr{\h}}\to \finstoch$ factorising $\cdsyn{\bullet}\to \finstoch$. This is the content of \cref{lem:triangulation_finstoch} below, which therefore allows us to say that the triangulation sends $(\omega,\h)\in \mn$ to $(\omega,\tr{\h})\in \bn$.

	As in the proof of \cref{thm:moralisation}, to describe how triangulation acts on morphisms, we first prove that every order-preserving graph homomorphism $\alpha\colon \h'\to \h$ can be interpreted as an order-preserving graph homomorphism $\tr{\h'}\to \tr{\h}$. 
	The fact that the graph homomorphism is order-preserving ensures that this is true, indeed whenever $v\to w$ in $\tr{\h'}$, then there exists a path $v\edge w_1 \edge \dots \edge w_n=w$ with $w_i \ge w$. As both of these properties are preserved by $\alpha$, the same is true for $\alpha(v)$ and $\alpha(w)$, and therefore they either coincide or there exists a directed edge $\alpha(v)\to \alpha(w)$, as wanted. 
	This ensures that the triangulation functor can be defined on morphisms by sending $(\alpha,\eta)\colon (\omega,\h)\to (\omega',\h')$ to $(\alpha,\eta)\colon (\omega,\tr{\h})\to (\omega',\tr{\h'})$.
\end{proof}

The ordered DAG $\tr{\h}$ is such that for all vertices $u,v,w\in V_{\g}$, 
\begin{equation}\label{eq:triangulatedDAG}
	u,v\to w\text{ and }u\le v\qquad \implies \qquad u\to v.
\end{equation}
This is immediate as $u,v\to w$ ensures two paths with $w$
\[
u \edge w_1 \edge \dots \edge w_n = w = w'_m \edge \dots \edge w'_1\edge v,
\]
where $w_i,w'_j \ge w$, so $u\to v$.
\begin{lemma}\label{lem:triangulation_finstoch}
	Let $\g$ be an ordered DAG satisfying~\eqref{eq:triangulatedDAG}.
	If the following diagram 
	\[
	\begin{tikzcd}
		\cdsyn{\bullet}\ar[r,"!_{\g}"]\ar[d,"\omega"]&\cdsyn{\g} \ar[r,"\star_{\g}"]& \syn{\g}\ar[d,"F"]\\
		\finstoch\ar[r,"i"] & \finprojstoch & \mat\ar[l,"q" above]
	\end{tikzcd}
	\]
	commutes, where $F$ is a hypergraph functor, then there exists a factorisation $\cdsyn{\bullet}\xrightarrow{!_{\g}} \cdsyn{\g} \xrightarrow{G} \finstoch$ of $\omega$ for some CD-functor $G$.
\end{lemma}
This result is the only one that uses tools from the standard theory.
However, the proof can be modified and rewritten from the viewpoint of categorical probability by using the notions of conditionals and normalisations~\cite{dilavore2024partial,lorenz2023causalmodels}.
\begin{proof}
	For the sake of simplicity, we use \cref{rem:semantics_copycomp}, as well as the discussion in the proof of \cref{prop:irredundantmarkov_functor} on page \pageref{eq:Q}, to translate the statement to a factorisation of probability distributions.
	In this way, the statement reads as follows: Let $\g$ be an ordered DAG satisfying \eqref{eq:triangulatedDAG} and consider a probability distribution $\omega$ and a factorisation $\omega(V_{\g})=\frac{1}{Z} \prod_{v\in V_{\g}} f_v (v\given \parents{v})$, where $f_v$ belongs to $\mat$. Then $\omega$ admits a factorisation $\omega = \prod_{v\in V_{\g}} g_v (v\given \parents{v})$ where $g_v$ is a stochastic matrix.

	We proceed by induction. Let us assume that the statement holds true for all graphs with $n-1$ vertices, and prove it for graphs with $n$ vertices.
	
	Let us take the biggest element $v\in V_{\g}$ and consider 
	\[ 
	\lambda_v(\parents{v}) \coloneqq \sum_{x\in v} f_v (x\given \parents{v})\quad \text{ and }\quad g_v(v\given \parents{v})\coloneqq \begin{cases}
			\frac{f_v (v\given \parents{v})}{\lambda_v(\parents{v})} & \text{if }\lambda_v(\parents{v}) \neq 0\\
			\operatorname{unif}_v & \text{otherwise}
		\end{cases}
	\]
	where $\operatorname{unif}_v$ is the uniform probability on $v$ (this can actually be substituted by any probability distribution). 
	In particular, with this choice $g_v$ is a stochastic matrix, i.e.\ it belongs to $\finstoch$.

	We note that $f_v (v\given \parents{v})= \lambda_v (\parents{v}) g_v(v\given \parents{v})$. 
	By assumption, if we take the biggest parent $w$ of $v$, then $\parents{v}\setminus\lbrace w \rbrace \subseteq \parents{w}$.
	In particular, we can define
	\begin{gather*}
		f'_w (w\given \parents{w})\coloneqq \lambda_v (\parents{v}) f_w (w\given \parents{w}),\qquad \text{and} \\
		f'_u(u \given \parents{u})\coloneqq f_u (u\given \parents{u})\qquad \text{for all }u\in V_{\g}\setminus \lbrace v,w\rbrace
	\end{gather*}
	Let $\g'$ be the graph obtained from $\g$ by deleting $v$. From the factorisation $\omega(V_{\g})=\frac{1}{Z}\prod_{u \in V_{\g'}} f'_u(u\given \parents{u}) g_v(v \given \parents{v})$, a direct check shows that $\frac{1}{Z}\prod_{u \in V_{\g'}} f'_u(u\given \parents{u})$ is equal to the marginalization $\omega'$ of $\omega$.
	Since $\g'$ satisfies~\eqref{eq:triangulatedDAG} because $\g$ does, we can apply the induction hypothesis to show that $\omega'$ can be factorised by stochastic matrices.
	Since $g_v(v\given \parents{v})$ is also a stochastic matrix, the statement is also true for $\omega$, as wanted.
\end{proof}

\begin{proof}[Proof of \cref{prop:trmor_mortr}]
	The two natural transformations are simply achieved by noting that the identity on an ordered DAG $\g$ and an ordered undirected graph $\h$ are also morphisms $\g\to \tr{\mor{\g}}$ and $\h\to\mor{\tr{\h}}$. 
	That these correspond to natural transformation follows because the functors $\mor{-}$ and $\tr{-}$ send the morphism $(\alpha,\eta)$ to the same pair intepreted in the new type, as stated at the end of the proofs of \cref{thm:moralisation,thm:triangulation}.
\end{proof}

\begin{proposition}\label{prop:noadjunction}
	There is no adjunction given by $\mor{-}$ and $\tr{-}$.
\end{proposition}
\begin{proof}
	By contradiction, let us assume that a natural transformation $\mu_{(\omega,\g)}\colon (\omega,\g)\to (\omega,\tr{\mor{\g}})$ exists. 
	Recall that the functors $\mor{-}$ and $\tr{-}$ are `identity-like', i.e.\ they send $(\alpha,\eta)$ to $(\alpha,\eta)$ intepreted in the new type. 
	For any $(\omega,\g)$ and any vertex $v \in \g$, let us consider $(\omega_v,\bullet)$ where $\omega_v$ is the marginalization of $\omega$ at $v$.
	The marginalization gives rise to a distribution-preserving monoidal transformation $\pi\colon \omega \to \omega_v$, so in particular we have a morphism $(i,\pi)\colon (\omega,\g)\to (\omega_v,\bullet)$ where $i$ is the inclusion of graphs $\bullet \to \g$ sending $\bullet$ to $v$.
	By assumption,
	\[
	\begin{tikzcd}[column sep=large]
		(\omega_v,\bullet)\ar[r,"\mu_{(\omega_v,\bullet)}"]\arrow[d,"{(i,\pi)}"] & (\omega_v,\bullet)\arrow[d,"{(i,\pi)}"]\\
		(\omega,\g)\ar[r,"\mu_{(\omega,\g)}"] & (\omega,\tr{\mor{\g}})
	\end{tikzcd}
	\]
	commutes.
	At the level of graphs, $\mu_{(\omega_v,\bullet)}$ must be given by the identity, as it is the only morphism $\bullet \to \bullet$. 
	Therefore, whatever $\alpha \colon \tr{\mor{\g}}\to \g$ represents $\eta_{(\omega,\g)}$ at the level of graphs, it must respect the commutation of the diagram above, which means that $\alpha(v)=v$.
	The arbitrariety of $v$ implies that $\alpha$ must be the identity, but the identity is not necessarily a graph homomorphisms $\tr{\mor{\g}}\to \g$ because $\tr{\mor{\g}}$ has in general more edges than $\g$.
	For the sake of an example, let $\g\coloneqq \begin{tikzpicture}
	\begin{pgfonlayer}{nodelayer}
		\node [style=bw] (15) at (0, 0) {$A$};
		\node [style=bw] (29) at (2.5, 0) {$C$};
		\node [style=bw] (31) at (5, 0) {$B$};
	\end{pgfonlayer}
	\begin{pgfonlayer}{edgelayer}
		\draw [style=arrow] (15) to (29);
		\draw [style=arrow] (31) to (29);
	\end{pgfonlayer}
\end{tikzpicture}
$, and note that $\tr{\mor{\g}}$ is a complete DAG (i.e.\ we have the additional edge $A \to B$). 

	The same idea applies when assuming the existence of a natural transformation $(\omega,\h)\to (\omega,\mor{\tr{\h}})$, and therefore $\mor{-}$ and $\tr{-}$ do not admit a unit for the possible adjunction in either direction. 
\end{proof}

\end{document}